\documentclass[10pt,journal,a4paper,compsoc]{IEEEtran}

\usepackage{url}
\usepackage{color}
\usepackage{textcomp}
\usepackage{multirow}
\usepackage[cmex10]{amsmath}
\usepackage{amssymb}
\usepackage{dsfont}
\usepackage{stfloats}
\usepackage{amsthm}
\usepackage{hyperref}

\usepackage{ifpdf}

\ifCLASSOPTIONcompsoc
   \usepackage[nocompress]{cite}
\else
   \usepackage{cite}
\fi
\ifCLASSINFOpdf
\usepackage[pdftex]{graphicx}
 \graphicspath{{./figures_v13/}{./biography/} }
\DeclareGraphicsExtensions{.pdf,.jpeg,.png}
\else
\usepackage[dvips]{graphicx}
   \graphicspath{{./figures_v13/}{./newfigures/}{./biography/} }
  \DeclareGraphicsExtensions{.eps}
\fi
\usepackage{algorithm}
\usepackage{algorithmic}

\ifCLASSOPTIONcompsoc
  \usepackage[font=normalsize,labelfont=sf,textfont=sf]{subfig}
\else
  \usepackage[caption=false]{caption}
  \usepackage[font=footnotesize]{subfig}
\fi

\usepackage{stfloats}

\newcommand{\tensor}[1]{\underline{\mathbf{#1}}}

\newtheorem{Proposition}{Proposition}[section]

\begin{document}
%
\title{Higher-Order Partial Least Squares (HOPLS): \\A Generalized Multi-Linear Regression Method }
%
%
%
%

\author{Qibin Zhao,
         Cesar F. Caiafa,
         Danilo P. Mandic,
         Zenas C. Chao,
         Yasuo Nagasaka,
        Naotaka Fujii,\\
        Liqing Zhang and
        Andrzej Cichocki 
\IEEEcompsocitemizethanks{\IEEEcompsocthanksitem Q. Zhao is with Laboratory for Advanced Brain Signal Processing, Brain Science Institute, RIKEN, Saitama, Japan.
\IEEEcompsocthanksitem C. F. Caiafa is with Instituto Argentino de Radioastronom\'{i}a (IAR), CCT La Plata - CONICET, Buenos Aires, Argentina.
\IEEEcompsocthanksitem Danilo P. Mandic is with Communication and Signal Processing Research Group, Department of Electrical and Electronic Engineering, Imperial College, London, UK
\IEEEcompsocthanksitem Z. C. Chao, Y. Nagasaka and N. Fujii are with Laboratory for Adaptive Intelligence, Brain Science Institute, RIKEN, Saitama, Japan.
\IEEEcompsocthanksitem L. Zhang is with MOE-Microsoft Laboratory for Intelligent Computing and Intelligent Systems and
Department of Computer Science and Engineering, Shanghai Jiao Tong University, Shanghai, China.
\IEEEcompsocthanksitem A. Cichocki is with Laboratory for Advanced Brain Signal Processing, Brain Science Institute, RIKEN, Japan  and Systems Research Institute in Polish Academy of Science.

}
\thanks{} }

\IEEEcompsoctitleabstractindextext{%
\begin{abstract}

A new generalized multilinear regression model, termed the Higher-Order Partial Least Squares (HOPLS), is introduced with the aim to predict a tensor (multiway array) $\tensor{Y}$ from a tensor $\tensor{X}$ through projecting the data onto the latent space and performing regression on the corresponding latent variables. HOPLS differs substantially from other regression models in that it explains the data by a sum of orthogonal Tucker tensors, while the number of orthogonal loadings serves as a parameter to control model complexity and prevent overfitting. The low dimensional latent space is optimized sequentially via a deflation operation, yielding the best joint subspace approximation for both $\tensor{X}$ and $\tensor{Y}$. Instead of decomposing $\tensor{X}$ and $\tensor{Y}$ individually, higher order singular value decomposition on a newly defined generalized cross-covariance tensor is employed to optimize the orthogonal loadings. A systematic comparison on both synthetic data and real-world decoding of 3D movement trajectories from electrocorticogram (ECoG) signals demonstrate the advantages of HOPLS over the existing methods in terms of better predictive ability, suitability to handle small sample sizes, and robustness to noise.

\end{abstract}

}

\maketitle

\IEEEdisplaynotcompsoctitleabstractindextext

%
\IEEEpeerreviewmaketitle

\section{Introduction}
\IEEEPARstart{T}{he} Partial Least Squares (PLS) is a well-established framework for estimation, regression and classification, whose objective is to predict a set of dependent variables (responses) from a set of independent variables (predictors) through the extraction of a small number of latent variables. One member of the PLS family is Partial Least Squares Regression (PLSR) - a multivariate method which, in contrast to Multiple Linear Regression (MLR) and Principal Component Regression (PCR), is proven to be particularly suited to highly collinear data \cite{Montgomery2001, dhanjal2009efficient}. In order to predict response variables $\mathbf{Y}$ from independent variables $\mathbf{X}$, PLS finds a set of latent variables (also called latent vectors, score vectors or components) by projecting both $\mathbf{X}$ and $\mathbf{Y}$ onto a new subspace, while at the same time maximizing the pairwise covariance between the latent variables of $\mathbf{X}$ and $\mathbf{Y}$. A standard way to optimize the model parameters is the Non-linear Iterative Partial Least Squares (NIPALS)\cite{wold1975soft}; for an overview of PLS and its applications in neuroimaging see \cite{krishnan2010partial, abdi2010partial, rosipal2006overview}. There are many variations of the PLS model including the orthogonal projection on latent structures (O-PLS) \cite{trygg2002orthogonal}, Biorthogonal PLS (BPLS) \cite{ergon2002pls}, recursive partial least squares (RPLS) \cite{vijayakumar2000locally}, nonlinear PLS \cite{rosipal2002kernel, ham2004kernel}. The PLS regression is known to exhibit high sensitivity to noise, a problem that can be attributed to redundant latent variables\cite{bro2005standard}, whose selection still remains an open problem \cite{li2002model}. Penalized regression methods are also popular for simultaneous variable selection and coefficient estimation, which impose e.g., L2 or L1 constraints on the regression coefficients. Algorithms of this kind are Ridge regression and Lasso \cite{tibshirani1996regression}. The recent progress in sensor technology, biomedicine, and biochemistry has highlighted the necessity to consider multiple data streams as multi-way data structures \cite{bro1998multi}, for which the corresponding analysis methods are very naturally based on tensor decompositions \cite{kolda2009tensor,Cichocki2009,acar2010scalable}. Although matricization of a tensor is an alternative way to express such data, this would result in the \textquotedblleft  Large $p$ Small $n$\textquotedblright  problem and also make it difficult to interpret the results, as the physical meaning and multi-way data structures would be lost due to the unfolding operation.

The $N$-way PLS (N-PLS) decomposes the independent and dependent data into rank-one tensors, subject to maximum pairwise covariance of the latent vectors. This promises enhanced stability, resilience to noise, and intuitive interpretation of the results \cite{bro1996multiway, bro2006review}. Owing to these desirable properties N-PLS has found applications in areas ranging from chemometrics\cite{hasegawa2000rational,  nilsson1997multiway, zissis1999two} to neuroscience \cite{martinez2004concurrent, acar2007seizure}. A modification of the N-PLS and the multi-way covariates regression were studied in \cite{bro2001difference, smilde1999multiway,smilde2004multi}, where the weight vectors yielding the latent variables are optimized by the same strategy as in N-PLS, resulting in better fitting to independent data $\tensor{X}$ while maintaining no difference in predictive performance. The tensor decomposition used within N-PLS is Canonical Decomposition /Parallel Factor Analysis (CANDECOMP/PARAFAC or CP) \cite{Harshman1970}, which makes N-PLS inherit both the advantages and limitations of CP \cite{smilde1997comments}. These limitations are related to poor fitness ability, computational complexity and slow convergence when handling multivariate dependent data and higher order ($N>3$) independent data, causing N-PLS not to be guaranteed to outperform standard PLS \cite{zissis1999two, B903649K}.

In this paper, we propose a new generalized mutilinear regression model, called Higer-Order Partial Least Squares (HOPLS), which makes it possible to predict an $M$th-order tensor $\tensor{Y}$ ($M\geq 3$) (or a particular case of two-way matrix $\mathbf{Y}$) from an $N$th-order tensor $\underline{\mathbf{X}}(N\geq 3)$ by projecting tensor $\underline{\mathbf{X}}$ onto a low-dimensional common latent subspace. The latent subspaces are optimized sequentially through simultaneous rank-$(1, L_2,\ldots, L_N)$ approximation of $\underline{\mathbf{X}}$ and rank-$(1, K_2,\ldots, K_M)$ approximation of $\tensor{Y}$ (or rank-one approximation in particular case of two-way matrix $\mathbf{Y}$). Owing to the better fitness ability of the orthogonal Tucker model as compared to CP \cite{kolda2009tensor} and the flexibility of the block Tucker model \cite{de2008decompositions}, the analysis and simulations show that HOPLS proves to be a promising multilinear subspace regression framework that provides not only an optimal tradeoff between fitness and model complexity but also enhanced predictive ability in general. In addition, we develop a new strategy to find a closed-form solution by employing higher-order singular value decomposition (HOSVD) \cite{de2000multilinear}, which makes the computation more efficient than the currently used iterative way.

The article is structured as follows. In Section \ref{sec:history}, an overview of two-way PLS is presented, and the notation and notions related to multi-way data analysis are introduced. In Section \ref{sec:HOPLS}, the new multilinear regression model is proposed, together with the corresponding solutions and algorithms. Extensive simulations on synthetic data and a real world case study on the fusion of behavioral and neural data are presented in Section \ref{results}, followed by conclusions in Section \ref{conclusions}.

\section{Background and Notation}
\label{sec:history}

\subsection{Notation and definitions} \label{subsec:notation}
$N$th-order tensors (\emph{multi-way arrays}) are denoted by underlined boldface capital letters, matrices (\emph{two-way arrays)} by boldface capital letters, and vectors by boldface lower-case letters. The $i$th entry of a vector $\mathbf{x}$ is denoted by $x_i$, element $(i, j)$ of a matrix $\mathbf{X}$ is denoted by $x_{ij}$, and element $(i_1, i_2, \ldots, i_N)$ of an $N$th-order tensor $\underline{\mathbf{X}}\in\mathds{R}^{I_{1}\times I_{2}\times\cdots\times I_{N} }$ by $x_{i_{1}i_{2}\ldots i_{N}}$ or $(\underline{\mathbf{X}})_{i_{1}i_{2}\ldots i_{N}}$. Indices typically range from $1$ to their capital version, e.g., $i_{N}=1, \ldots, I_{N}$. The mode-$n$ matricization of a tensor is denoted by $\mathbf{X}_{(n)}\in \mathds{R} ^{I_n\times I_1 \cdots I_{n-1} I_{n+1} \cdots I_N}$. The $n$th factor matrix in a sequence is denoted by  $\mathbf{A}^{(n)}$.

The \emph{$n$-mode product} of a tensor $\underline{\mathbf{X}}\in\mathds{R}^{I_{1}\times\cdots \times I_{n}\times \cdots \times I_{N} }$ and matrix $\mathbf{A}\in\mathds{R}^{J_n \times I_n}$is denoted by $\underline{\mathbf{Y}}=\underline{\mathbf{X}}\times_n\mathbf{A} \in \mathds{R}^{I_1\times\cdots\times I_{n-1}\times J_n\times I_{n+1}\times\cdots\times I_N}$ and is defined as:
\begin{equation}
 y_{i_1i_2\ldots i_{n-1}j_{n}i_{n+1}\ldots i_N}=\sum_{i_n}x_{i_1i_2\ldots i_n\ldots i_N}a_{j_ni_n}.
\end{equation}

The \emph{rank-$(R_1,R_2,...,R_N)$ Tucker model} \cite{Tucker1963} is a tensor decomposition defined and denoted as follows:
\begin{multline}\label{Tucker}
   \underline{\mathbf{Y}} \approx \underline{\mathbf{G}}\times_1\mathbf{A}^{(1)} \times_2 \mathbf{A}^{(2)}\times_3 \cdots \times_N \mathbf{A}^{(N)}\\
   = [\![ \tensor{G}; \mathbf{A}^{(1)},\ldots,\mathbf{A}^{(N)}]\!],
\end{multline}
where $\underline{\mathbf{G}}\in{\mathds{R}^{R_1\times R_2\times..\times R_N}}, (R_n\leq I_n)$ is the \emph{core tensor} and $\mathbf{A}^{(n)}\in{\mathds{R}^{I_n\times R_n}}$ are the \emph{factor matrices}. The last term  is the simplified notation, introduced in \cite{kolda2006multilinear}, for the Tucker operator. When the factor matrices are orthonormal and the core tensor is all-orthogonal this model is called HOSVD \cite{de2000multilinear, kolda2006multilinear}.

The \emph{CP model} \cite{Harshman1970,Carroll1970,de2000best,Silva2008,kolda2009tensor} became prominent in Chemistry \cite{smilde2004multi} and is defined as a sum of rank-one tensors:
\begin{equation}\label{PARAFAC}
   \underline{\mathbf{Y}}\approx \sum_{r=1}^{R}\lambda_r \mathbf{a}^{(1)}_r \circ \mathbf{a}^{(2)}_r\circ\cdots\circ \mathbf{a}^{(N)}_r,
\end{equation}
where the symbol `$\circ$'    denotes the outer product of vectors, $\mathbf{a}^{(n)}_r$ is the column-$r$ vector of matrix $\mathbf{A}^{(n)}$, and $\lambda_r$ are scalars. The CP model can also be represented by (\ref{Tucker}), under the condition that the core tensor is super-diagonal, i.e., $R_1=R_2=\cdots=R_N$ and $g_{i_1i_2,...,i_N}=0$ if $ i_n\neq i_m$ for all $n\neq m$.


The $1$-mode product between $\underline{\mathbf{G}}\in\mathds{R}^{1\times I_2\times\cdots\times I_N}$ and $\mathbf{t}\in\mathds{R}^{I_1\times 1}$ is of size $I_1\times I_2\times \cdots\times I_N$, and is defined as
\begin{equation}\label{Eq:mode1product}
(\underline{\mathbf{G}} \times_1 \mathbf{t})_{i_{1} i_{2}\ldots i_{N}} = g_{1 i_{2}\ldots i_{N}} t_{i_{1}}.
\end{equation}

The \emph{inner product of two tensors} $\underline{\mathbf{A}},\underline{\mathbf{B}}\in{\mathds{R}^{I_1\times I_2...\times I_N}}$ is defined by $\langle \underline{\mathbf{A}}, \underline{\mathbf{B}}\rangle = \sum_{i_1i_2...i_N} a_{i_1i_2...i_N}b_{i_1i_2...i_N}$, and the squared Frobenius norm by $\|\underline{\mathbf{A}}\|_F^2 = \langle \underline{\mathbf{A}}, \underline{\mathbf{A}}\rangle$.

The \emph{$n$-mode cross-covariance} between an $N$th-order tensor $\underline{\mathbf{X}}\in\mathds{R}^{I_{1}\times \cdots\times I_{n}\times\cdots\times I_{N} }$ and an $M$th-order tensor $\underline{\mathbf{Y}}\in\mathds{R}^{J_{1}\times \cdots\times I_{n} \times\cdots\times J_{M}}$ with the same size $I_n$ on the $n$th-mode, denoted by
$\mbox{COV}_{\{n;n\}}(\underline{\mathbf{X}},\underline{\mathbf{Y}})\in\mathds{R}^{I_1\times\cdots\times I_{n-1} \times I_{n+1}\times\cdots \times I_N \times J_1 \times \cdots\times J_{n-1} \times J_{n+1}\times \cdots\times J_M }$, is defined as
\begin{equation}\label{Eq:HOPLS-covxy2}
\tensor{C}=\mbox{COV}_{\{n;n\}}(\underline{\mathbf{X}},\underline{\mathbf{Y}}) = <\underline{\mathbf{X}}, \underline{\mathbf{Y}}>_{\{n;n\}},
\end{equation}
where the symbol $<\bullet, \bullet>_{\{n;n\}}$ represents an $n$-mode multiplication between two tensors, and is defined as
\begin{multline}\label{Eq:prodtensor}
c_{i_1,\ldots,i_{n-1},i_{n+1}\ldots i_N, j_1,\ldots,j_{n-1}j_{n+1}\ldots j_M} =\\
\sum_{i_n=1}^{I_n} x_{i_1,\ldots,i_n,\ldots,i_{N}} y_{j_1,\ldots,i_n,\ldots,j_M}.
\end{multline}
As a special case, for a matrix $\mathbf{Y}\in\mathds{R}^{I_{n}\times M}$, the \emph{$n$-mode cross-covariance} between $\underline{\mathbf{X}}$ and $\mathbf{Y}$ simplifies as
\begin{equation}\label{Eq:HOPLS-covxy2}
\mbox{COV}_{\{n;1\}}(\underline{\mathbf{X}},\mathbf{Y}) = \underline{\mathbf{X}}\times_{n}\mathbf{Y}^\mathrm{T},
\end{equation}
under the assumption that $n$-mode column vectors of $\underline{\mathbf{X}}$ and columns of $\mathbf{Y}$ are mean-centered.

\subsection{Standard PLS (two-way PLS)}
\begin{figure}[htbp]
\centering
\includegraphics[width=0.45\textwidth]{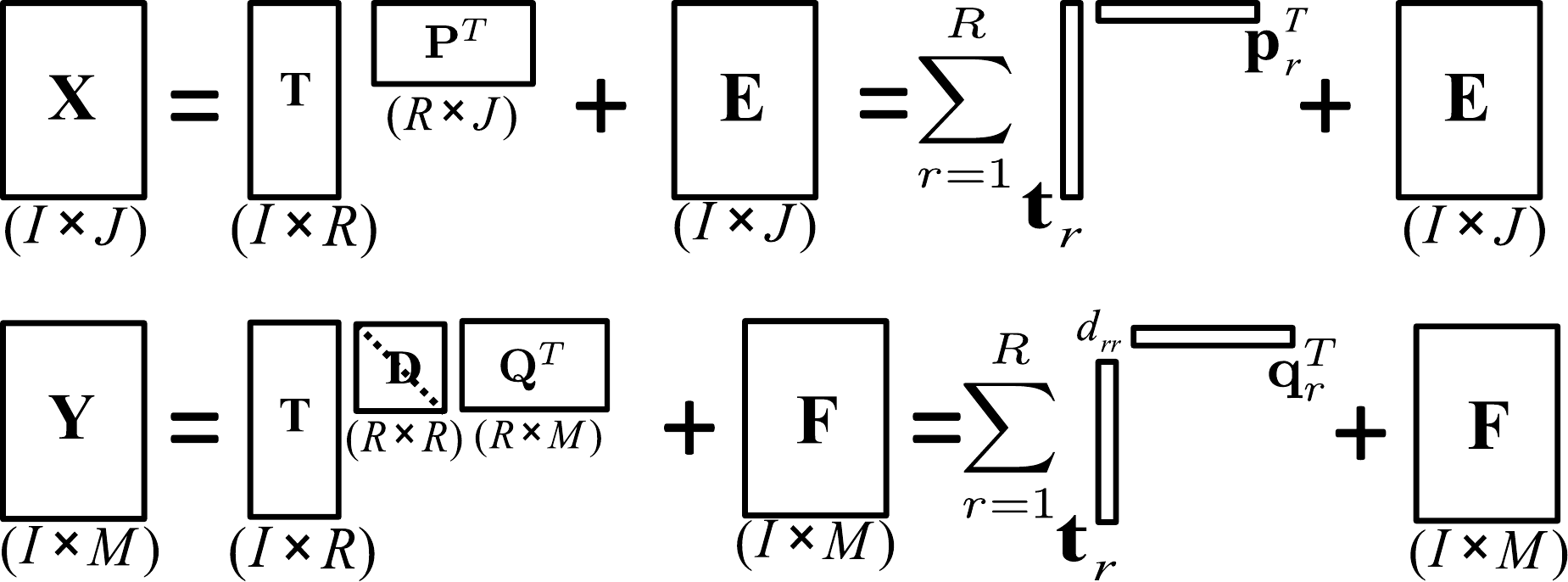}
\caption{\small The PLS model: data decomposition as a sum of rank-one matrices.}
\label{fig:PLS}
\end{figure}

The PLS regression was originally developed for econometrics by H. Wold\cite{wold1975soft, wold1982soft} in order to deal with collinear predictor variables. The usefulness of PLS in chemical applications was illuminated by the group of S. Wold\cite{wold2001pls, wold1984collinearity}, after some initial work by Kowalski \textit{et al}. \cite{kowalski1982systems}. Currently, the PLS regression is being widely applied in chemometrics, sensory evaluation, industrial process control, and more recently, in the analysis of functional brain imaging data\cite{ mcintosh2004partial, mcintosh2004spatiotemporal, kovacevic2007groupwise, chao2010long, trejo2006brain}.

The principle behind PLS is to search for a set of latent vectors by performing a simultaneous decomposition of $\mathbf{X}\in\mathds{R}^{I\times J}$ and $\mathbf{Y}\in\mathds{R}^{I\times M}$ with the constraint that these components explain as much as possible of the covariance between $\mathbf{X}$ and $\mathbf{Y}$. This can be formulated as
\begin{eqnarray}\label{Eq:PLS-X}
 \mathbf{X} &=& \mathbf{T}\mathbf{P}^{T} + \mathbf{E} = \sum_{r=1}^{R}{\mathbf{t}_{r}\mathbf{p}_{r}^{T}} + \mathbf{E},\\
\label{Eq:PLS-Y}
 \mathbf{Y} &=& \mathbf{U}\mathbf{Q}^{T} + \mathbf{F} = \sum_{r=1}^{R}{\mathbf{u}_{r}\mathbf{q}_{r}^{T}} + \mathbf{F},
\end{eqnarray}
where $\mathbf{T}=[\mathbf{t}_1, \mathbf{t}_2, \ldots , \mathbf{t}_R]\in\mathds{R}^{I\times R}$ consists of $R$ extracted orthonormal latent variables from $\mathbf{X}$, i.e. $\mathbf{T}^T\mathbf{T}=\mathbf{I}$, and  $\mathbf{U} =[\mathbf{u}_1, \mathbf{u}_2,\ldots ,\mathbf{u}_R]\in\mathds{R}^{I\times R}$ are latent variables from $\mathbf{Y}$ having maximum covariance with $\mathbf{T}$ column-wise.  The matrices $\mathbf{P}$ and $\mathbf{Q}$ represent loadings and $\mathbf{E}, \mathbf{F}$ are respectively the residuals for $\mathbf{X}$ and $\mathbf{Y}$. In order to find the first set of components, we need to optimize the two sets of weights $\mathbf{w}, \mathbf{q}$ so as to satisfy
\begin{equation}\label{Eq:PLSoptimization}
\max_{\{\mathbf{w}, \mathbf{q}\}} [\mathbf{w}^{T}\mathbf{X}^{T}\mathbf{Y}\mathbf{q}], \quad \mbox{s. t. }\quad \mathbf{w}^{T}\mathbf{w} = 1, \mathbf{q}^{T}\mathbf{q} = 1.
\end{equation}
The latent variable then is estimated as $\mathbf{t}=\mathbf{X}\mathbf{w}$. Based on the assumption of a linear relation $\mathbf{u} \approx d \,\mathbf{t}$, $\mathbf{Y}$ is predicted by
\begin{equation}\label{Eq:PLS-YY}
\mathbf{Y} \approx \mathbf{T}\mathbf{D}\mathbf{Q}^{T},
\end{equation}
where $\mathbf{D}$ is a diagonal matrix with $d_{rr} = \mathbf{u}_{r}^{T}\mathbf{t}_{r}/\mathbf{t}_{r}^{T}\mathbf{t}_{r}$, implying that the problem boils down to finding common latent variables $\mathbf{T}$ that explain the variance of both $\mathbf{X}$ and $\mathbf{Y}$, as illustrated in Fig. {\ref{fig:PLS}}.

\section{Higher-order PLS (HOPLS)}
\label{sec:HOPLS}
\begin{figure}[htbp]
\centering
\includegraphics[width=0.48\textwidth]{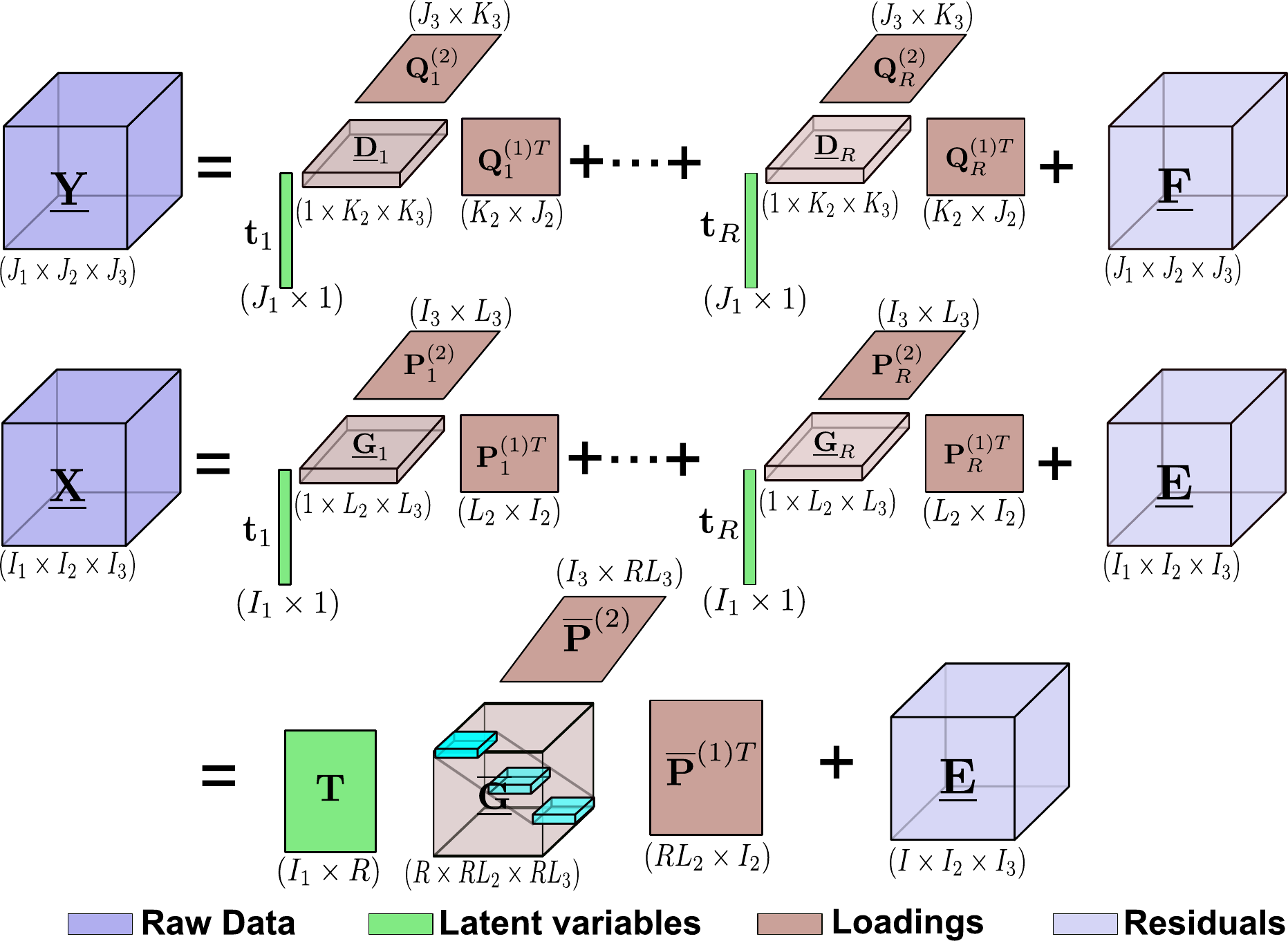}
\caption{\small Schematic diagram of the HOPLS model: approximating $\underline{\mathbf{X}}$ as a sum of rank-$(1, L_2, L_3)$ tensors. Approximation for $\underline{\mathbf{Y}}$ follows a similar principle with shared common latent components $\mathbf{T}$.}
\label{fig:HOPLS}
\end{figure}


For a two-way matrix, the low-rank approximation is equivalent to subspace approximation, however, for a higher-order tensor, these two criteria lead to completely different models (i.e., CP and Tucker model). The $N$-way PLS (N-PLS), developed by Bro \cite{bro1996multiway}, is a straightforward multi-way extension of standard PLS based on the CP model. Although CP model is the best low-rank approximation, Tucker model is the best subspace approximation, retaining the maximum amount of variation \cite{bro2001difference}. It thus provides better fitness than the CP model except in a special case when perfect CP exists, since CP is a restricted version of the Tucker model when the core tensor is super-diagonal.

There are two different approaches for extracting the latent components: sequential and simultaneous methods. A sequential method extracts one latent component at a time, deflates the proper tensors and calculates the next component from the residuals. In a simultaneous method, all components are calculated simultaneously by minimizing a certain criterion. In the following, we employ a sequential method since it provides better performance.

Consider an $N$th-order independent tensor $\underline{\mathbf{X}}\in\mathds{R}^{I_1\times\cdots\times I_N}$ and an $M$th-order dependent tensor $\underline{\mathbf{Y}}\in\mathds{R}^{J_1\times\cdots\times J_M}$, having the same size on the first mode, i.e., $I_1 = J_1$. Our objective is to find the optimal subspace approximation of $\underline{\mathbf{X}}$ and $\underline{\mathbf{Y}}$, in which the latent vectors from $\tensor{X}$ and $\tensor{Y}$ have maximum pairwise covariance. Considering a linear relation between the latent vectors, the problem boils down to finding the common latent subspace which can approximate both $\tensor{X}$ and $\tensor{Y}$ simultaneously. We firstly address the general case of a tensor $\tensor{X} (N\geq3)$ and a tensor $\tensor{Y}(M\geq3)$. A particular case with a tensor $\tensor{X}(N\geq3)$ and a matrix $\mathbf{Y}(M=2)$ is presented separately in Sec. \ref{sec:t2m}, using a slightly different approach.

\subsection{Proposed model}
\label{sec:t2t}
Applying Tucker decomposition within a PLS framework is not straightforward, and to that end we propose a novel block-wise orthogonal Tucker approach to model the data. More specifically, we assume $\underline{\mathbf{X}}$ is decomposed as a sum of rank-($1, L_2, \ldots, L_N$) Tucker blocks, while $\underline{\mathbf{Y}}$ is decomposed as a sum of rank-($1, K_2, \ldots, K_M$) Tucker blocks (see Fig. \ref{fig:HOPLS}), which can be expressed as
\begin{equation}\label{Eq:HOPLS_T2T}
\begin{split}
\underline{\mathbf{X}} &= \sum_{r=1}^{R} \underline{\mathbf{G}}_{r} \! \times_{1} \mathbf{t}_{r}\! \times_{2}\mathbf{P}_{r}^{(1)}\! \times_{3}\cdots\! \times_{N}\mathbf{P}_{r}^{(N-1)} +\!  \underline{\mathbf{E}}_R,\\
\underline{\mathbf{Y}} &= \sum_{r=1}^{R}\underline{\mathbf{D}}_{r}\!  \times_{1} \mathbf{t}_{r}\! \times_{2}\mathbf{Q}_{r}^{(1)}\! \times_{3}\cdots\! \times_M \! \mathbf{Q}_{r}^{(M-1)} + \! \underline{\mathbf{F}}_R,
\end{split}
\end{equation}
where $R$ is the number of latent vectors, $\mathbf{t}_{r}\in\mathds{R}^{I_1}$ is the $r$-th latent vector, $\left\{\mathbf{P}^{(n)}_{r}\right\}_{n=1}^{N-1}\in\mathds{R}^{I_{n+1}\times L_{n+1}}$ and $\left\{\mathbf{Q}^{(m)}_{r}\right\}_{m=1}^{M-1}\in\mathds{R}^{J_{m+1}\times K_{m+1}} $ are loading matrices on mode-$n$ and mode-$m$ respectively, and $\underline{\mathbf{G}}_{r}\in\mathds{R}^{1\times L_2\times\cdots\times L_N}$ and $\underline{\mathbf{D}}_{r}\in\mathds{R}^{1\times K_2\times\cdots\times K_M}$ are core tensors.


However the Tucker decompositions in (\ref{Eq:HOPLS_T2T}) are not unique \cite{kolda2009tensor} due to the permutation, rotation, and scaling issues. To alleviate this problem, additional constraints should be imposed such that the core tensors $\underline{\mathbf{G}}_{r}$ and $\underline{\mathbf{D}}_{r}$ are all-orthogonal, a sequence of loading matrices are column-wise orthonormal, i.e., \mbox{$\mathbf{P}^{(n)T}_{r}\mathbf{P}^{(n)}_{r} =\mathbf{I}$} and \mbox{$\mathbf{Q}^{(m)T}_{r}\mathbf{Q}^{(m)}_{r}=\mathbf{I}$}, the latent vector is of length one, i.e. $\|\mathbf{t}_r\|_F=1$. Thus, each term in (\ref{Eq:HOPLS_T2T}) is represented as an orthogonal Tucker model, implying essentially uniqueness as it is subject only to trivial indeterminacies \cite{de2008decompositions}.

By defining a latent matrix $\mathbf{T} = [\mathbf{t}_{1},\ldots,\mathbf{t}_{R}]$, \mbox{mode-$n$} loading matrix $\overline{\mathbf{P}}^{(n)} =[\mathbf{P}^{(n)}_{1},\ldots,\mathbf{P}^{(n)}_{R}]$, mode-$m$ loading matrix $\overline{\mathbf{Q}}^{(m)} =[\mathbf{Q}^{(m)}_{1},\ldots,\mathbf{Q}^{(m)}_{R}]$ and core tensor $\overline{\underline{\mathbf{G}}}= \mbox{blockdiag}(\underline{\mathbf{G}}_{1},\ldots,\underline{\mathbf{G}}_{R})\in\mathds{R}^{R\times RL_2\times\cdots\times RL_N}$, $\overline{\underline{\mathbf{D}}}= \mbox{blockdiag}(\underline{\mathbf{D}}_{1},\ldots,\underline{\mathbf{D}}_{R})\in\mathds{R}^{R\times RK_2\times\cdots\times RK_M}$, the HOPLS model in (\ref{Eq:HOPLS_T2T}) can be rewritten as
\begin{equation}\label{Eq:HOPLS-model-T}
\begin{split}
   \underline{\mathbf{X}} &= \overline{\underline{\mathbf{G}}}\times_{1}\mathbf{T}\times_{2}\overline{\mathbf{P}}^{(1)}\times_3\cdots\times_{N}\overline{\mathbf{P}}^{(N-1)}+\underline{\mathbf{E}}_R, \\
      \underline{\mathbf{Y}} &= \overline{\underline{\mathbf{D}}}\times_{1}\mathbf{T}\times_{2}\overline{\mathbf{Q}}^{(1)}\times_3\cdots\times_M\overline{\mathbf{Q}}^{(M-1)}+\underline{\mathbf{F}}_R,
\end{split}
\end{equation}
where $\underline{\mathbf{E}}_R$ and $\underline{\mathbf{F}}_R$ are residuals after extracting $R$ components. The core tensors $\overline{\underline{\mathbf{G}}}$ and $\overline{\underline{\mathbf{D}}}$ have a special block-diagonal structure (see Fig. \ref{fig:HOPLS}) and their elements indicate the level of local interactions between the corresponding latent vectors and loading matrices. Note that the tensor decomposition in (\ref{Eq:HOPLS-model-T}) is similar to the block term decomposition discussed in \cite{de2008decompositions}, which aims to the decomposition of only one tensor. However, HOPLS attempts to find the block Tucker decompositions of two tensors with block-wise orthogonal constraints, which at the same time satisfies a certain criteria related to having common latent components on a specific mode.

Benefiting from the advantages of Tucker decomposition over the CP model \cite{kolda2009tensor}, HOPLS promises to approximate data better than N-PLS. Specifically, HOPLS differs substantially from the N-PLS model in the sense that extraction of latent components in HOPLS is based on subspace approximation rather than on low-rank approximation and the size of loading matrices is controlled by a hyperparameter, providing a tradeoff between fitness and model complexity. Note that HOPLS simplifies into \mbox{N-PLS} if we define $\forall n: \{L_n\}=1$ and $\forall m: \{K_m\}=1$.

\subsection{Optimization criteria and algorithm}
\label{optimization_criteria}

The tensor decompositions in (\ref{Eq:HOPLS_T2T}) consists of two simultaneous optimization problems: (i) approximating $\tensor{X}$ and $\tensor{Y}$ by orthogonal Tucker model, (ii) having at the same time a common latent component on a specific mode. If we apply HOSVD individually on $\tensor{X}$ and $\tensor{Y}$, the best rank-($1,L_2,\ldots,L_N$) approximation for $\tensor{X}$ and the best rank-($1,K_2,\ldots,K_M$) approximation for $\tensor{Y}$ can be obtained while the common latent vector $\mathbf{t}_r$ cannot be ensured.
Another way is to find the best approximation of $\tensor{X}$ by HOSVD first, subsequently, $\tensor{Y}$ can be approximated by a fixed $\mathbf{t}_r$. However, this procedure, which resembles multi-way principal component regression \cite{smilde2004multi}, has the drawback that the common latent components are not necessarily predictive for $\tensor{Y}$.

The optimization of subspace transformation according to (\ref{Eq:HOPLS_T2T}) will be formulated as a problem of determining a set of orthogonormal loadings $\mathbf{P}^{(n)}_{r},\mathbf{Q}^{(m)}_{r}, r=1,2,\ldots, R$ and latent vectors $\mathbf{t}_r$ that satisfies a certain criterion. Since each term can be optimized sequentially with the same criteria based on deflation, in the following, we shall simplify the problem to that of finding the first latent vector $\mathbf{t}$ and two sequences of loading matrices $\mathbf{P}^{(n)}$ and $\mathbf{Q}^{(m)}$.

In order to develop a strategy for the simultaneous minimization of the Frobenius norm of residuals $\mathbf{E}$ and $\mathbf{F}$, while keeping a common latent vector $\mathbf{t}$, we first need to introduce the following basic results:
\begin{Proposition}\label{Proposition1}
Given a tensor $\underline{\mathbf{X}}\in \mathds{R}^{I_1 \times \cdots \times I_N }$ and column orthonormal matrices $\mathbf{P}^{(n)}\in\mathds{R}^{I_{n+1}\times L_{n+1}}, n=1,\ldots,N-1$, $\mathbf{t}\in\mathds{R}^{I_1}$ with $\|\mathbf{t}\|_F=1$, the least-squares (LS) solution to $\min_{\underline{\mathbf{G}}} \|\underline{\mathbf{X}} - \underline{\mathbf{G}}\times_1 \mathbf{t} \times_2 \mathbf{P}^{(1)} \times_3 \cdots\times_N \mathbf{P}^{(N-1)} \|^2_F$ is given by $\underline{\mathbf{G}}=\underline{\mathbf{X}}\times_1 \mathbf{t}^T \times_2 \mathbf{P}^{(1)T} \times_3 \cdots \times_N \mathbf{P}^{(N-1)T}$.
\end{Proposition}

\begin{proof}
\hangindent \leftmargini
This result is very well known and is widely used in the literature \cite{kolda2009tensor, de2000multilinear}. A simple proof is based on writing the mode-1 matricization of tensor $\underline{\mathbf{X}}$ as
\begin{equation}
\mathbf{X}_{(1)} = \mathbf{t}\mathbf{G}_{(1)}(\mathbf{P}^{(N-1)}\otimes \cdots\otimes \mathbf{P}^{(1)})^T + \mathbf{E}_{(1)},
\end{equation}
where tensor $\underline{ \mathbf{E}}_{(1)}$ is the residual and the symbol `$\otimes$' denotes the \emph{Kronecker product}. Since $\mathbf{t}^T\mathbf{t}=1$ and $(\mathbf{P}^{(N-1)}\otimes \cdots\otimes \mathbf{P}^{(1)})$ is column orthonormal, the LS solution of $\mathbf{G}_{(1)}$ with fixed matrices $\mathbf{t}$ and $\mathbf{P}^{(n)}$ is given by  $\mathbf{G}_{(1)} = \mathbf{t}^T\mathbf{X}_{(1)}(\mathbf{P}^{(N-1)}\otimes \cdots\otimes \mathbf{P}^{(1)})$; writing it in a tensor form we obtain the desired result. \qedhere
\end{proof}

\begin{Proposition} \label{Proposition2}
Given a fixed tensor $\underline{\mathbf{X}}\in \mathds{R}^{I_1\times\cdots\times I_N}$, the following two constrained optimization problems are equivalent:

1) $\min_{\{\mathbf{P}^{(n)},\mathbf{t}\}} \|\underline{\mathbf{X}} - \underline{\mathbf{G}}\times_1 \mathbf{t}\! \times_2\mathbf{P}^{(1)}\! \times_3\cdots\! \times_{N}\mathbf{P}^{(N-1)} \|^2_F$,
s. t. matrices $\mathbf{P}^{(n)}$ are column orthonormal and $\|\mathbf{t}\|_F =1$.

2) $\max_{\{\mathbf{P^{(n)}},\mathbf{t}\}} \|\underline{\mathbf{X}}\times_1 \mathbf{t}^T \times_2 \mathbf{P}^{(1)T} \times_3 \cdots\! \times_{N}\mathbf{P}^{(N-1)T} \|^2_F$, s. t. matrices $\mathbf{P}^{(n)}$ are column orthonormal and $\|\mathbf{t}\|_F =1$.
\end{Proposition}
The proof is available in \cite{kolda2009tensor} (see pp. 477-478).

Assume that the orthonormal matrices $\mathbf{P}^{(n)}, \mathbf{Q}^{(m)}, \mathbf{t}$ are given, then from Proposition \ref{Proposition1}, the core tensors in (\ref{Eq:HOPLS_T2T}) can be computed as
\begin{equation}\label{Eq:coreGD}
\begin{split}
\tensor{G} &= \tensor{X}\times_1 \mathbf{t}^T \times_2 \mathbf{P}^{(1)T} \times_3 \cdots\! \times_{N}\mathbf{P}^{(N-1)T},\\
\tensor{D} &= \tensor{Y}\times_1 \mathbf{t}^T \times_2 \mathbf{Q}^{(1)T} \times_3 \cdots\! \times_{M}\mathbf{Q}^{(M-1)T}.
\end{split}
\end{equation}
According to Proposition \ref{Proposition2}, minimization of $\|\tensor{E}\|_F$ and $\|\tensor{F}\|_F$ under the orthonormality constraint is equivalent to maximization of $\|\tensor{G}\|_F$ and $\|\tensor{D}\|_F$.

However, taking into account the common latent vector $\mathbf{t}$ between $\tensor{X}$ and $\tensor{Y}$, there is no straightforward way to maximize $\|\tensor{G}\|_F$ and $\|\tensor{D}\|_F$ simultaneously. To this end, we propose to maximize a product of norms of two core tensors, i.e., $\max \{\|\tensor{G}\|_F^2 \cdot \|\tensor{D}\|_F^2\}$. Since the latent vector $\mathbf{t}$ is determined by $\mathbf{P}^{(n)},\mathbf{Q}^{(m)}$, the first step is to optimize the orthonormal loadings, then the common latent vectors can be computed by the fixed loadings.

\begin{Proposition}\label{Proposition3}
Let $\tensor{G}\in\mathds{R}^{1\times L_2 \times\cdots\times L_N}$ and $\tensor{D}\in\mathds{R}^{1\times K_2 \times\cdots\times K_M}$, then $\|<\tensor{G},\tensor{D}>_{\{1;1\}}\|^2_F =  \|\tensor{G}\|^2_F \cdot \|\tensor{D}\|^2_F$.
\end{Proposition}
\begin{proof}
\hangindent \leftmargini
\begin{align}
\nonumber
\|<\tensor{G},&\tensor{D}>_{\{1;1\}}\|_F^2 =  \|\mbox{vec}(\tensor{G})\mbox{vec}^{T}(\tensor{D})\|_F^2\\
\nonumber
&= \mbox{trace}\left(\mbox{vec}(\tensor{D})\mbox{vec}^{T}(\tensor{G}) \mbox{vec}(\tensor{G})\mbox{vec}^{T}(\tensor{D})^{T}  \right)\\
& = \|\mbox{vec}(\tensor{G}) \|_F^2 \cdot \|\mbox{vec}(\tensor{D}) \|_F^2.
\end{align}
where $\mbox{vec}(\tensor{G})\in \mathds{R}^{L_2L_3\ldots L_N}$ is the vectorization of the tensor $\tensor{G}$. \qedhere
\end{proof}

From Proposition \ref{Proposition3}, observe that to maximize \mbox{$\|\tensor{G}\|_F^2 \cdot \|\tensor{D}\|_F^2$} is equivalent to maximizing \mbox{$\|<\tensor{G},\tensor{D}>_{\{1;1\}}\|_F^2$}. According to (\ref{Eq:coreGD}) and $\mathbf{t}^{T}\mathbf{t}=1$, \mbox{$\|<\tensor{G},\tensor{D}>_{\{1;1\}} \|_F^{2}$} can be expressed as
\begin{equation}\label{Eq:HOPLS-newobj1}
\small
\left \|[\![ <\tensor{X},\tensor{Y}>_{\{1;1\}}; \mathbf{P}^{(1)T}\!\!\!,\ldots,\!\mathbf{P}^{(N-1)T},\mathbf{Q}^{(1)T}\!\!\!\!,\!\ldots,\!\mathbf{Q}^{(M-1)T}   ]\!]\right\|_F^{2}.
\end{equation}
Note that this form is quite similar to the optimization problem for two-way PLS in (\ref{Eq:PLSoptimization}), where the cross-covariance matrix $\mathbf{X}^T\mathbf{Y}$ is replaced by \mbox{$<\tensor{X},\tensor{Y}>_{\{1;1\}}$}. In addition, the optimization item becomes the norm of a small tensor in contrast to a scalar in (\ref{Eq:PLSoptimization}). Thus, if we define \mbox{$<\tensor{X},\tensor{Y}>_{\{1;1\}}$ as a mode-$1$} cross-covariance tensor $\underline{\mathbf{C}}=\mbox{COV}_{\{1;1\}}(\underline{\mathbf{X}},\underline{\mathbf{Y}})\in\mathds{R}^{I_2\times\cdots\times I_N \times J_2 \times \cdots\times J_M}$, the optimization problem can be finally formulated as
\begin{align}
\label{Eq:HOPLS-finalobj}
\nonumber
   \!\! \!\!\max_{\left\{\mathbf{P^{(n)}},\mathbf{Q^{(m)}}\right\}} &\left \| [\![ \underline{\mathbf{C}}; \mathbf{P}^{(1)T},\!\ldots,\!\mathbf{P}^{(N-1)T},\mathbf{Q}^{(1)T},\!\ldots\!,\mathbf{Q}^{(M-1)T}   ]\!]\right\|_F^{2}  \\
 \mbox{s. t.} \quad & \mathbf{P}^{(n)T}\mathbf{P}^{(n)} =\mathbf{I}_{L_{n+1}}, \mathbf{Q}^{(m)T}\mathbf{Q}^{(m)} = \mathbf{I}_{K_{m+1}},
\end{align}
where $\mathbf{P}^{(n)}, n=1,\ldots,N-1$ and $\mathbf{Q}^{(m)},  m=1,\ldots,M-1$ are the parameters to optimize.

Based on Proposition \ref{Proposition2} and orthogonality of $\mathbf{P}^{(n)},\mathbf{Q}^{(m)}$, the optimization problem in (\ref{Eq:HOPLS-finalobj}) is equivalent to find the best subspace approximation of $\tensor{C}$ as
\begin{equation}
\tensor{C}  \approx [\![ \tensor{G}^{(C)}; \mathbf{P}^{(1)},\ldots,\mathbf{P}^{(N-1)},\mathbf{Q}^{(1)},\ldots,\mathbf{Q}^{(M-1)}]\!],
\end{equation}
which can be obtained by \mbox{rank-($L_2,\ldots,L_N,K_2,\ldots,K_M$)} HOSVD on tensor $\underline{\mathbf{C}}$. Based on Proposition \ref{Proposition1}, the optimization term in (\ref{Eq:HOPLS-finalobj}) is  equivalent to the norm of core tensor $\tensor{G}^{(C)}$. To achieve this goal, the higher-order orthogonal iteration (HOOI) algorithm \cite{de2000best, kolda2009tensor}, which is known to converge fast, is employed to find the parameters $\mathbf{P}^{(n)}$ and $\mathbf{Q}^{(m)}$ by orthogonal Tucker decomposition of $\underline{\mathbf{C}}$.

Subsequently, based on the estimate of the loadings $\mathbf{P}^{(n)}$ and $\mathbf{Q}^{(m)}$, we can now compute the common latent vector $\mathbf{t}$. Note that taking into account the asymmetry property of the HOPLS framework, we need to estimate $\mathbf{t}$ from predictors $\underline{\mathbf{X}}$ and to estimate regression coefficient $\tensor{D}$ for prediction of responses $\tensor{Y}$.
For a given set of loading matrices $\{\mathbf{P}^{(n)}\}$, the latent vector $\mathbf{t}$ should explain variance of $\underline{\mathbf{X}}$ as much as possible, that is
\begin{equation}\label{Eq:HOPLS-argmaxt}
\mathbf{t}=\arg\min_{\mathbf{t}} \left\|\underline{\mathbf{X}} - [\![\underline{\mathbf{G}};\mathbf{t},\mathbf{P}^{(1)},\ldots,\mathbf{P}^{(N-1)}]\!] \right\|_F^{2},
\end{equation}
which can be easily achieved by choosing $\mathbf{t}$ as the first leading left singular vector of the matrix \mbox{$(\tensor{X}\times_2\mathbf{P}^{(1)T}\times_3\cdots\times_N\mathbf{P}^{(N-1)T})_{(1)}$} as used in the HOOI algorithm (see \cite{kolda2009tensor,kolda2006multilinear}).
Thus, the core tensors $\tensor{G}$ and $\tensor{D}$ are computed by (\ref{Eq:coreGD}).

\renewcommand{\algorithmicrequire}{\textbf{Input:}}
\renewcommand{\algorithmicensure}{\textbf{Output:}}
\begin{algorithm}[t]
\caption{\small{The Higher-order Partial Least Squares (HOPLS) Algorithm for a Tensor $\tensor{X}$ and a Tensor $\tensor{Y}$}}
\label{alg1}
\begin{algorithmic}
\REQUIRE $\underline{\mathbf{X}}\in\mathbb{R}^{I_{1}\times\cdots\times I_{N}}, \underline{\mathbf{Y}}\in\mathbb{R}^{J_{1}\times\cdots\times J_{M}}$, $N\geq3, M\geq3$ and $I_1 = J_1$.
\STATE Number of latent vectors is $R$ and number of loading vectors are $\{L_{n}\}_{n=2}^N$ and $\{K_{m}\}_{m=2}^M$.
\ENSURE $\{\mathbf{P}^{(n)}_{r}\}; \{\mathbf{Q}^{(m)}_{r}\}; \{\underline{\mathbf{G}}_{r}\}; \{\underline{\mathbf{D}}_{r}\}; \mathbf{T}$
\STATE $r=1,\ldots, R; \, n=1,\ldots, N-1; \, m=1,\ldots,M-1$.
\STATE \textbf{Initialization:} $\underline{\mathbf{E}}_{1} \leftarrow \underline{\mathbf{X}}, \quad \underline{\mathbf{F}}_{1} \leftarrow \underline{\mathbf{Y}}$.
\FOR{$r=1$ {\bfseries to}  $R$ }
\IF{ $\|\underline{\mathbf{E}}_{r}\|_F >\varepsilon \; \AND\; \| \underline{\mathbf{F}}_{r}\|_F >\varepsilon$}
\STATE $\underline{\mathbf{C}}_{r} \leftarrow <\underline{\mathbf{E}}_{r}, \underline{\mathbf{F}}_{r}>_{\{1,1\}}$;
\STATE Rank-$(L_{2},\ldots, L_{N},K_{2},\ldots,K_{M})$ orthogonal Tucker decomposition of $\underline{\mathbf{C}}_{r}$ by HOOI \cite{kolda2009tensor} as
\STATE $\underline{\mathbf{C}}_{r} \approx  [\![\underline{\mathbf{G}}_r^{(C_r)}; \mathbf{P}_r^{(1)},\ldots,\mathbf{P}_r^{(N-1)},\mathbf{Q}_r^{(1)},\ldots,\mathbf{Q}_r^{(M-1)}  ]\!]$;
\STATE $\mathbf{t}_r \leftarrow \mbox{the first leading left singular vector by}$  \\
 $\qquad \mbox{SVD}\left[\left(\underline{\mathbf{E}}_r \times_{2}\mathbf{P}_r^{(1)T}\times_{3}\cdots\times_N \mathbf{P}_r^{(N-1)T}\right)_{(1)}\right]$;
\STATE $\underline{\mathbf{G}}_r \leftarrow [\![\underline{\mathbf{E}}_r; \mathbf{t}_r^{T},\mathbf{P}_r^{(1)T},\ldots,\mathbf{P}_r^{(N-1)T} ]\!]$;
\STATE $\underline{\mathbf{D}}_r \leftarrow [\![\underline{\mathbf{F}}_r; \mathbf{t}_r^{T},\mathbf{Q}_r^{(1)T},\ldots,\mathbf{Q}_r^{(M-1)T} ]\!]$;
\STATE \textbf{Deflation}:
\STATE  $ \underline{\mathbf{E}}_{r+1} \leftarrow \underline{\mathbf{E}}_{r} - [\![\underline{\mathbf{G}}_r; \mathbf{t}_r,\mathbf{P}_r^{(1)},\ldots,\mathbf{P}_r^{(N-1)} ]\!]$;
\STATE $ \underline{\mathbf{F}}_{r+1} \leftarrow \underline{\mathbf{F}}_{r} - [\![\underline{\mathbf{D}}_r; \mathbf{t}_r,\mathbf{Q}_r^{(1)},\ldots,\mathbf{Q}_r^{(M-1)} ]\!]$;
\ELSE
\STATE Break;
\ENDIF
\ENDFOR
\end{algorithmic}
\end{algorithm}

The above procedure should be carried out repeatedly using the deflation operation, until an appropriate number of components (i.e., $R$) are obtained, or the norms of residuals are smaller than a certain threshold. The deflation\footnote{Note that the latent vectors are not orthogonal in HOPLS algorithm, which is related to deflation. The theoretical explanation and proof are given in the supplement material.} is performed by subtracting from $\underline{\mathbf{X}}$ and $\tensor{Y}$ the information explained by a rank-($1,L_2,\ldots,L_N$) tensor $\widehat{\tensor{X}}$ and a rank-($1,K_2,\ldots,K_M$) tensor $\widehat{\tensor{Y}}$, respectively. The HOPLS algorithm is outlined in Algorithm \ref{alg1}.

\subsection{The case of the tensor $\tensor{X}$ and matrix $\mathbf{Y}$ }
\label{sec:t2m}


Suppose that we have an $N$th-order independent tensor $\underline{\mathbf{X}}\in\mathds{R}^{I_1\times\cdots\times I_N}$ ($N\geq 3$) and a two-way dependent data $\mathbf{Y}\in\mathds{R}^{I_1\times M}$, with the same sample size $I_1$. Since for two-way matrix, subspace approximation is equivalent to low-rank approximation. HOPLS operates by modeling independent data $\underline{\mathbf{X}}$  as a sum of rank-($1, L_2, \ldots, L_N$) tensors while dependent data $\mathbf{Y}$ is modeled with a sum of rank-one matrices as
\begin{align}
\label{Eq:HOPLS_T2M}
\mathbf{Y} &= \sum_{r=1}^{R}d_{r}\mathbf{t}_{r}\mathbf{q}_{r}^{T} + \mathbf{F}_R,
\end{align}
where $\|\mathbf{q}_r\|=1$ and $d_r$ is a scalar.



\begin{Proposition}\label{Proposition4}
Let $\tensor{Y}\in\mathds{R}^{I\times M}$ and $\mathbf{q}\in\mathds{R}^{M}$ is of length one, then $\mathbf{t} = \mathbf{Yq}$ solves the problem $\min_{t} \|\mathbf{Y} - \mathbf{t}\mathbf{q}^T \|_F^2$. In other words, a linear combination of the columns of $\mathbf{Y}$ by using a weighting vector $\mathbf{q}$ of length one has least squares properties in terms of approximating $\mathbf{Y}$.
\end{Proposition}
\begin{proof}
\hangindent \leftmargini
Since $\mathbf{q}$ is given and $\|\mathbf{q}\|=1$, it is obvious that the ordinary least squares solution to solve the problem is $\mathbf{t} = \mathbf{Y}\mathbf{q}(\mathbf{q}^{T}\mathbf{q})^{-1}$, hence, $\mathbf{t}=\mathbf{Yq}$. If a $\mathbf{q}$ with length one is found according to some criterion, then automatically $\mathbf{tq}^T$ with $\mathbf{t}=\mathbf{Yq}$ gives the best fit of $\mathbf{Y}$ for that $\mathbf{q}$.
 \qedhere
\end{proof}

As discussed in the previous section, the problem of minimizing $\|\underline{\mathbf{E}}\|^2_F$ with respect to matrices $\mathbf{P}^{(n)}$ and vector $\mathbf{t}\in \mathds{R}^{I}$ is equivalent to maximizing the norm of core tensor $\tensor{G}$ with an orthonormality constraint. Meanwhile, we attempt to find an optimal $\mathbf{q}$ with unity length which ensures that $\mathbf{Yq}$ is linearly correlated with the latent vector $\mathbf{t}$, i.e., $d\mathbf{t} = \mathbf{Yq} $, then according to Proposition \ref{Proposition4}, $d\mathbf{t}\mathbf{q}^T$ gives the best fit of $\mathbf{Y}$. Therefore, replacing $\mathbf{t}$ by $d^{-1}\mathbf{Y}\mathbf{q}$ in the expression for the core tensor $\tensor{G}$ in (\ref{Eq:coreGD}), we can optimize the parameters of X-loading matrices $\mathbf{P}^{(n)}$ and Y-loading vector $\mathbf{q}$  by maximizing the norm of $\tensor{G}$, which gives the best approximation of both tensor $\tensor{X}$ and matrix $\mathbf{Y}$. Finally, the optimization problem of our interest can be formulated as:
\begin{align}\label{Eq:HOPLS-T2M-Obj}
\nonumber
\max_{\{\mathbf{P}^{(n)},\mathbf{q}\}} & \|\underline{\mathbf{X}}\times_1 \mathbf{Y}^T \times_1 \mathbf{q}^T \times_2 \mathbf{P}^{(1)T}\! \times_{3}\cdots\! \times_{N}\mathbf{P}^{(N-1)T} \|_F^2,\\
\mbox{s. t.} \quad & \mathbf{P}^{(n)T}\mathbf{P}^{(n)}=\mathbf{I}, \|\mathbf{q}\|_F =1.
\end{align}
where the loadings $\mathbf{P}^{(n)}$ and $\mathbf{q}$ are parameters to optimize. This form is similar to (\ref{Eq:HOPLS-finalobj}), but has a different cross-covariance tensor $\underline{\mathbf{C}} = \underline{\mathbf{X}}\times_{1}\mathbf{Y}^\mathrm{T}$ defined between a tensor and a matrix, implying that the problem can be solved by performing a \mbox{rank-($1, L_2, \ldots, L_N$)} HOSVD on $\underline{\mathbf{C}}$. Subsequently, the core tensor $\tensor{G}^{(C)}$ corresponding to $\tensor{C}$ can  also be computed.

Next, the latent vector $\mathbf{t}$ should be estimated so as to best approximate $\tensor{X}$ with given loading matrices $\mathbf{P}^{(n)}$. According to the model for $\tensor{X}$, if we take its mode-1 matricizacion, we can write
\begin{equation}
\mathbf{X}_{(1)} = \mathbf{t}\mathbf{G}_{(1)}(\mathbf{P}^{(N-1)T}\otimes \cdots \otimes \mathbf{P}^{(1)})^T + \mathbf{E}_{(1)},
\end{equation}
where \mbox{$\mathbf{G}_{(1)}\in\mathds{R}^{1\times L_2L_3\ldots L_N}$} is still unknown. However, the core tensor $\tensor{G}$ (i.e., $[\![\underline{\mathbf{X}}; \mathbf{t}^T, \mathbf{P}^{(1)T},\ldots ,\mathbf{P}^{(N-1)T}]\!]$)
and the core tensor $\tensor{G}^{(C)}$ (i.e., $[\![\underline{\mathbf{C}}; \mathbf{q}^T, \mathbf{P}^{(1)T},\ldots ,\mathbf{P}^{(N-1)T}]\!]$) has a linear connection that $\tensor{G}^{(C)} = d \tensor{G}$. Therefore, the latent vector $\mathbf{t}$ can be estimated in another way that is different with the previous approach in Section \ref{optimization_criteria}. For fixed matrices $\mathbf{G}_{(1)}=d^{-1}(\tensor{G}^{(C)})_{(1)}$, $\mathbf{X}_{(1)}$, $\mathbf{P}^{(n)}$ the least square solution for the normalized $\mathbf{t}$, which minimizes the squared norm of the residual $\|\mathbf{E}_{(1)}\|^2_F$, can be obtained from
\begin{equation}\label{Eq:HOPLS-computet}
\mathbf{t}\leftarrow (\underline{\mathbf{X}} \times_2 \mathbf{P}^{(1)T}\! \times_{3}\cdots\! \times_{N}\mathbf{P}^{(N-1)T})_{(1)}\mathbf{G}^{(C)+}_{(1)},\; \mathbf{t}\leftarrow\mathbf{t}/\|\mathbf{t}\|_F,
\end{equation}
where we used the fact that $\mathbf{P}^{(n)}$ are columnwise orthonormal and the symbol $\rq +\rq$  denotes \emph{Moore-Penrose pseudoinverse}. With the estimated latent vector $\mathbf{t}$, and loadings $\mathbf{q}$, the regression coefficient used to predict $\mathbf{Y}$ is computed as
 \begin{equation}
d= \mathbf{t}^{T}\mathbf{Yq}.
\end{equation}

The procedure for a two-way response matrix is summarized in Algorithm \ref{alg2}. In this case, HOPLS model is also shown to unify both standard PLS and N-PLS within the same framework, when the appropriate parameters $L_n$ are selected\footnote{Explanation and proof are given in the supplement material.}.

\renewcommand{\algorithmicrequire}{\textbf{Input:}}
\renewcommand{\algorithmicensure}{\textbf{Output:}}
\begin{algorithm}[t]
\caption{\small{Higher-order Partial Least Squares (HOPLS2) for a Tensor $\tensor{X}$ and a Matrix $\mathbf{Y}$}}
\label{alg2}
\begin{algorithmic}
\REQUIRE $\underline{\mathbf{X}}\in\mathbb{R}^{I_{1}\times I_{2}\times\cdots\times I_{N}}, N\geq 3$ and $\mathbf{Y}\in\mathbb{R}^{I_{1}\times M}$
\STATE The Number of latent vectors is $R$ and the number of loadings are $\{L_{n}\}_{n=2}^N$.
\ENSURE $\{\mathbf{P}^{(n)}_{r}\}; \mathbf{Q}; \{\underline{\mathbf{G}}_{r}\}; \mathbf{D}; \mathbf{T}; \mbox{$r=1,\ldots, R$}, \mbox{$n=2,\ldots, N$}$.
\STATE \textbf{Initialization:} $\underline{\mathbf{E}}_{1} \leftarrow \underline{\mathbf{X}}, \mathbf{F}_{1} \leftarrow \mathbf{Y}$.
\FOR{$r=1$ \TO $R$ }
\IF{ $\|\underline{\mathbf{E}}_{r}\|_F >\varepsilon  \; \AND\; \|\mathbf{F}_{r}\|_F  >\varepsilon $}
\STATE $\underline{\mathbf{C}}_{r} \leftarrow \underline{\mathbf{E}}_{r}\times_{1}\mathbf{F}^{T}_{r}$ ;
\STATE Perform rank-$(1, L_{2},\cdots, L_{N})$ HOOI on $\underline{\mathbf{C}}_{r}$ as
\STATE $\underline{\mathbf{C}}_{r} \approx  \underline{\mathbf{G}}^{(C)}_{r} \times_{1}\mathbf{q}_{r}\times_{2}\mathbf{P}^{(1)}_{r}\times_3\cdots\times_{N}\mathbf{P}^{(N-1)}_{r}$;
\STATE {\small $\mathbf{t}_r \leftarrow \!\! \left(\underline{\mathbf{E}}_{r}\times_{2}\mathbf{P}_{r}^{(1)}\times_3 \!\!\cdots\!\!\times_{N}\mathbf{P}_{r}^{(N-1)}\right)_{(1)}\!\!\left(\mbox{vec}^{T}(\underline{\mathbf{G}}^{(C)}_{r} )\right)^{+}$;}
\STATE $\mathbf{t}_r \leftarrow \mathbf{t}_r/\| \mathbf{t}_r\|_F$;
\STATE $\underline{\mathbf{G}}_r \leftarrow [\![\underline{\mathbf{E}}_r; \mathbf{t}_r^{T},\mathbf{P}_r^{(1)T},\ldots,\mathbf{P}_r^{(N-1)T} ]\!]$;
\STATE $\mathbf{u}_{r} \leftarrow \mathbf{F}_{r}\mathbf{q}_{r}$;
\STATE $d_{r} \leftarrow \mathbf{u}_{r}^{T}\mathbf{t}_{r}$;
\STATE \textbf{Deflation:}
\STATE  $ \underline{\mathbf{E}}_{r+1} \leftarrow \underline{\mathbf{E}}_{r} - [\![\underline{\mathbf{G}}_r; \mathbf{t}_r,\mathbf{P}_r^{(1)},\ldots,\mathbf{P}_r^{(N-1)} ]\!]$;
\STATE $\mathbf{F}_{r+1}\leftarrow \mathbf{F}_{r}-d_{r} \mathbf{t}_{r}\mathbf{q}_{r}^{T}$;
\ENDIF
\ENDFOR
\end{algorithmic}
\end{algorithm}

\subsection{Prediction of the Response Variables}

Predictions from the new observations $\underline{\mathbf{X}}^{new}$ are performed in two steps: projecting the data to the low-dimensional latent space based on model parameters $\tensor{G}_r$, $\mathbf{P}_r^{(n)}$, and predicting the response data based on latent vectors $\mathbf{T}^{new}$ and model parameters $\mathbf{Q}_r^{(m)}$, $\tensor{D}_r$. For simplicity, we use a matricized form to express the prediction procedure as
\begin{align}
\hat{\tensor{Y}}_{(1)}^{new}   &\approx \mathbf{T}^{new} \mathbf{Q}^{*T} = \mathbf{X}^{new}_{(1)} \mathbf{W} \mathbf{Q}^{*T},
\end{align}
where $\mathbf{W}$ and $\mathbf{Q}^{*}$ have $R$ columns, represented by
\begin{equation}
\begin{split}
\mathbf{w}_r &= \left( \mathbf{P}_r^{(N-1)}\otimes\cdots\otimes \mathbf{P}_r^{(1)}   \right) \underline{\mathbf{G}}_{r(1)}^{+},\\
\mathbf{q}_r^{*} & = \underline{\mathbf{D}}_{r(1)} \left( \mathbf{Q}_r^{(M-1)}\otimes\cdots\otimes \mathbf{Q}_r^{(1)}\right)^{T}.
\end{split}
\end{equation}
In the particular case of a two-way matrix $\mathbf{Y}$, the prediction is performed by
\begin{align}
\hat{\mathbf{Y}}^{new}  &\approx  \mathbf{X}^{new}_{(1)} \mathbf{W} \mathbf{D} \mathbf{Q}^{T},
\end{align}
where $\mathbf{D}$ is a diagonal matrix whose entries are $\mathbf{d}_r$ and $r$th column of $\mathbf{Q}$ is $\mathbf{q}_r$, $r=1,\ldots,R$.

\subsection{Properties of HOPLS}

\emph{Robustness to noise.} An additional constraint of keeping the largest $\{L_n\}_{n=2}^{N}$ loading vectors on each mode is imposed in HOPLS, resulting in a flexible model that balances the two objectives of fitness and the significance of associated latent variables. For instance, a larger $L_n$ may fit $\underline{\mathbf{X}}$ better but introduces more noise to each latent vector. In contrast, N-PLS is more robust due to the strong constraint of rank-one tensor structure, while lacking good fit to the data. The flexibility of HOPLS allows us to adapt the model complexity based on the dataset in hands, providing considerable prediction ability (see Fig. \ref{fig:simulation1}, \ref{fig:simulation2}).

\emph{``Large $p$, Small $n$'' problem.} This is particularly important when the dimension of independent variables is high. In contrast to PLS, the relative low dimension of model parameters that need to be optimized in HOPLS. For instance, assume that a 3th-order tensor $\underline{\mathbf{X}}$ has the dimension of $5\times 10 \times 100$, i.e., there are 5 samples and 1000 features. If we apply PLS on $\mathbf{X}_{(1)}$ with size of $5\times 1000$, there are only five samples available to optimize a 1000-dimensional loading vector $\mathbf{p}$, resulting in an unreliable estimate of model parameters. In contrast, HOPLS allows us to optimize loading vectors, having relatively low-dimension, on each mode alternately; thus the number of samples is significantly elevated. For instance, to optimize 10-dimensional loading vectors on the second mode, 500 samples are available, and to optimize the 100-dimensional loading vectors on the third mode there are 50 samples. Thus, a more robust estimate of low-dimensional loading vectors can be obtained, which is also less prone to overfitting and more suitable for ``Large $p$, Small $n$'' problem (see \mbox{Fig. \ref{fig:simulation1}}).

\emph{Ease of interpretation.} The loading vectors in $\mathbf{P}^{(n)}$ reveal new subspace patterns corresponding to the $n$-mode features. However, the loadings from Unfold-PLS are difficult to interpret since the data structure is destroyed by the unfolding operation and the dimension of loadings is relatively high.

\emph{Computation.} N-PLS is implemented by combining a NIPALS-like algorithm with the CP decomposition.
Instead of using an iterative algorithm, HOPLS can find the model parameters using a closed-form solution, i.e., applying HOSVD on the cross-covariance tensor, resulting in enhanced computational efficiency.

Due to the flexibility of HOPLS, the tuning parameters of $L_{n}$ and $K_{m}$, controlling the model complexity, need to be selected based on calibration data. Similarly to the parameter $R$, the tuning parameters can be chosen by cross-validation. For simplicity, two alternative assumptions will been utilized: a) $\forall n, \forall m, L_{n}=K_{m} = \lambda$; b) $L_{n}= \eta R_{n}, K_{m} = \eta R_{m}, 0< \eta \leqslant 1$, i.e., explaining the same percentage of the $n$-mode variance.

\section{Experimental Results}
\label{results}
In the simulations, HOPLS and N-PLS were used to model the data in a tensor form whereas PLS was performed on a mode-1 matricization of the same tensors. To quantify the predictability, the index $Q^{2}$ was defined as $Q^{2} = 1- \| \tensor{Y}-\hat{\tensor{Y}}\|_F^{2}/\|\tensor{Y}\|_F^{2}$, where $\hat{\tensor{Y}}$ denotes the prediction of $\tensor{Y}$ using a model created from a calibration dataset. Root mean square errors of prediction (RMSEP) were also used for evaluation\cite{kim2005three}.

\subsection{Synthetic data}
\label{simulations}
In order to quantitatively benchmark our algorithm against the state of the art, an extensive comparative exploration has been performed on synthetic datasets to evaluate the prediction performance under varying conditions with respect to data structure, noise levels and ratio of variable dimension to sample size. For parameter selection, the number of latent vectors ($R$) and number of loadings ($L_n=K_m=\lambda$) were chosen based on five-fold cross-validation on the calibration dataset. To reduce random fluctuations, evaluations were performed over 50 validation datasets generated repeatedly according to the same criteria.

\subsubsection{Datasets with matrix structure}
\label{simulations_PLS}
The independent data $\mathbf{X}$ and dependent data $\mathbf{Y}$ were generated as:
\begin{equation}\label{Eq:gendata-pls}
\mathbf{X} = \mathbf{T}\mathbf{P}^{T} + \xi{\mathbf{E}}, \qquad \mathbf{Y} = \mathbf{T}\mathbf{Q}^{T} + \xi\mathbf{F},
\end{equation}
where latent variables $\{\mathbf{t},\mathbf{p},\mathbf{q}\}\sim\mathcal{N}(0,1)$, $\mathbf{E}$, $\mathbf{F}$ are Gaussian noises whose level is controlled by the parameter $\xi$. Both the calibration and the validation datasets were generated according to (\ref{Eq:gendata-pls}), with the same loadings $\mathbf{P},\mathbf{Q}$, but a different latent $\mathbf{T}$ which follows the same distribution $\mathcal{N}(0,1)$. Subsequently, the datasets were reorganized as $N$th-order tensors.

To investigate how the prediction performance is affected by noise levels and small sample size, $\{\tensor{X},\tensor{Y}\}\in\mathds{R}^{20\times 10\times 10}$ (Case 1) and $\{\tensor{X},\tensor{Y}\}\in\mathds{R}^{10\times 10\times 10}$ (Case 2) were generated under varying noise levels of 10dB, 5dB, 0dB and -5dB. In the case 3, $\{\tensor{X},\tensor{Y}\}\in\mathds{R}^{10\times 10\times 10}$ were generated with the loadings $\mathbf{P},\mathbf{Q}$ drawn from a uniform distribution $U(0,1)$. The datasets were generated from five latent variables (i.e., $\mathbf{T}$ has five columns) for all the three cases.

\begin{figure}[htbp]
\centering
\includegraphics[width=0.5\textwidth]{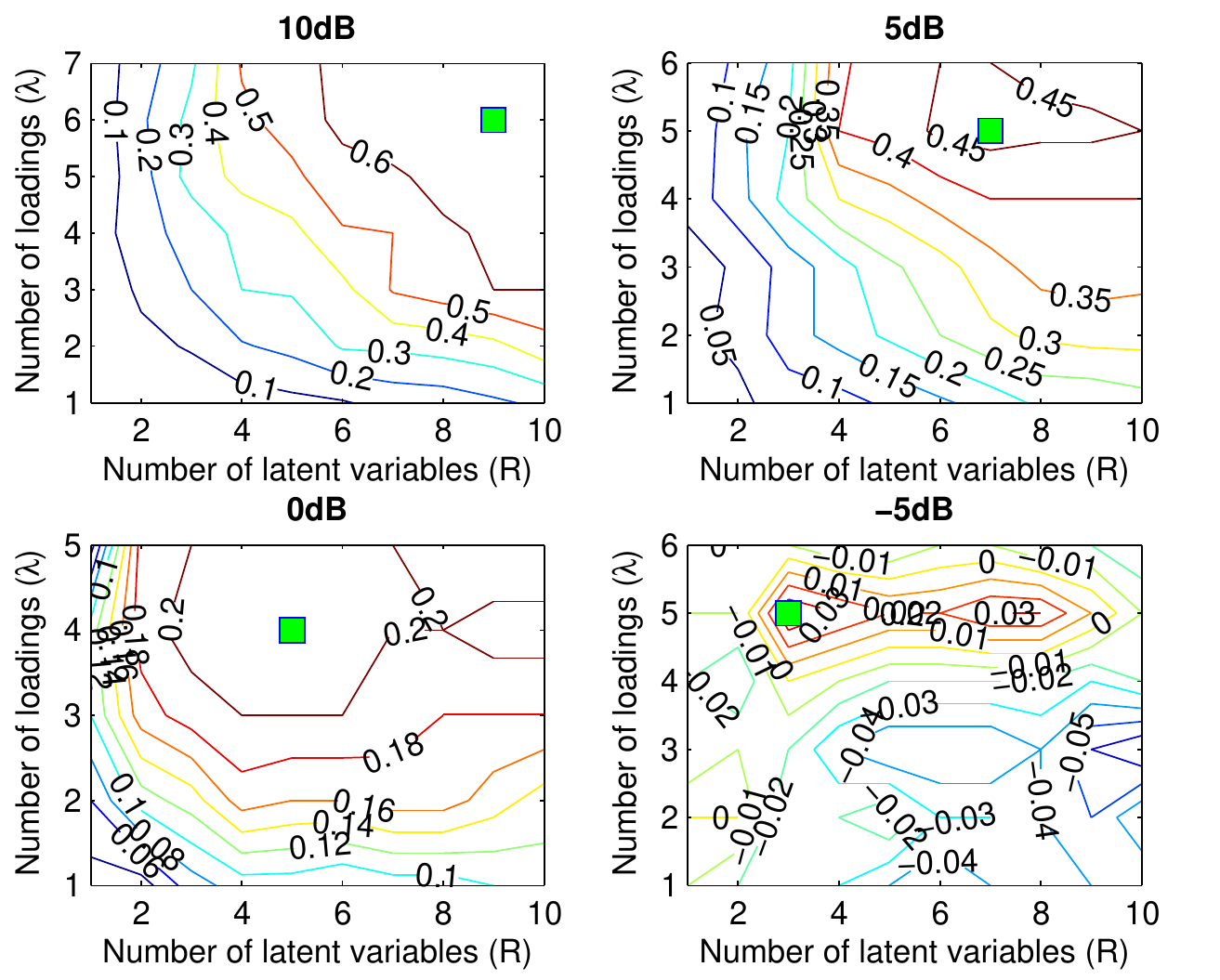}
\caption{\small Five-fold cross-validation performance of HOPLS at different noise levels versus the number of latent variables ($R$) and loadings ($\lambda$). The optimal values for these two parameters are marked by green squares. }
\label{fig:parameter1}
\end{figure}

\begin{table}[htbp]
\caption{\small The selection of parameters $R$ and $\lambda$ in Case 2.}
\renewcommand{\arraystretch}{1.3}
\label{table:parameters1}
\centering
\begin{tabular}{c c c c c c c c c c }
  \hline
  \multirow{2}{*}{SNR} & \multirow{2}{*}{PLS} & \multirow{2}{*}{N-PLS} & \multicolumn{2}{c}{HOPLS} & \multirow{2}{*}{SNR} & \multirow{2}{*}{PLS} & \multirow{2}{*}{N-PLS} & \multicolumn{2}{c}{HOPLS}\\
  \cline{4-5} \cline{9-10}
  & & & $R$ &$\lambda$   & & & & $R$ &$\lambda$\\
    \hline
  10dB & 5 & 7 & 9 &6 &0dB & 3 & 5 & 5 &4\\
  5dB & 5 & 6 & 7 &5&-5dB & 3 & 1 & 3 &5\\
  \hline
\end{tabular}
\end{table}

There are two tuning parameters, i.e., number of latent variables $R$ and number of loadings $\lambda$ for HOPLS and only one parameter $R$ for PLS and N-PLS, that need to be selected appropriately. The number of latent variables $R$ is crucial to prediction performance, resulting in under-modelling when $R$ was too small while overfitting easily when $R$ was too large. The cross-validations were performed when $R$ and $\lambda$ were varying from 1 to 10 with the step length of 1. In order to alleviate the computation burden, the procedure was stopped when the performance starts to decrease with increasing $\lambda$. Fig. \ref{fig:parameter1} shows the grid of cross-validation performance of HOPLS in Case 2 with the optimal parameters marked by green squares. Observe that the optimal  $\lambda$ for HOPLS is related to the noise levels, and for increasing noise levels, the best performance is obtained by smaller $\lambda$, implying that only few significant loadings on each mode are kept in the latent space. This is expected, due to the fact that the model complexity is controlled by $\lambda$ to suppress noise. The optimal $R$ and $\lambda$ for all three methods at different noise levels are shown in \mbox{Table \ref{table:parameters1}}.

\begin{figure}[htbp]
\centering
\includegraphics[width=0.45\textwidth]{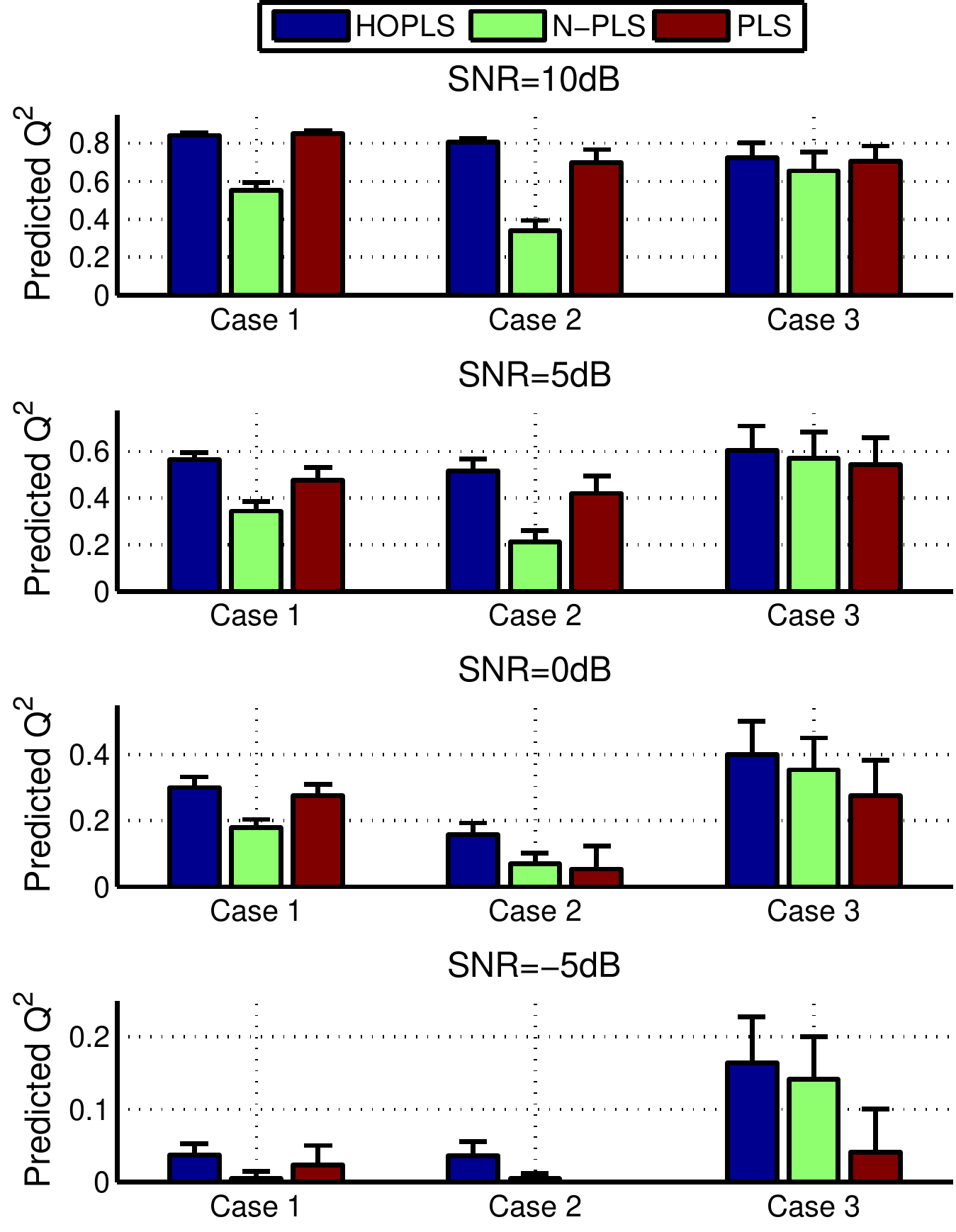}
\caption{\small The prediction performance comparison among HOPLS, N-PLS and PLS at different noise levels for three cases. Case1: $\{\tensor{X},\tensor{Y}\}\in\mathds{R}^{20\times 10\times 10}$ and $\{\mathbf{P,Q}\} \sim\mathcal{N}(0,1) $; Case 2: $\{\tensor{X},\tensor{Y}\}\in\mathds{R}^{10\times 10\times 10}$ and $\{\mathbf{P,Q}\}\sim\mathcal{N}(0,1)$; Case 3: $\{\tensor{X},\tensor{Y}\}\in\mathds{R}^{10\times 10\times 10}$ and $\{\mathbf{P,Q}\}\sim U(0,1)$. }
\label{fig:simulation1}
\end{figure}

After the selection the parameters, HOPLS, N-PLS and PLS are re-trained on the whole calibration dataset using the optimal $R$ and $\lambda$, and were applied to the validation datasets for evaluation. Fig. \ref{fig:simulation1} illustrates the predictive performance over 50 validation datasets for the three cases at four different noise levels. In Case 1, a relatively larger sample size was available, when SNR=10dB, HOPLS achieved a similar prediction performance to PLS while outperforming N-PLS. With increasing the noise level in both the calibration and validation datasets, HOPLS showed a relatively stable performance whereas the performance of PLS decreased significantly. The superiority of HOPLS was shown clearly with increasing the noise level. In Case 2 where a smaller sample size was available, HOPLS exhibited better performance than the other two models and the superiority of HOPLS was more pronounced at high noise levels, especially for SNR$\leq$5dB. These results demonstrated that HOPLS is more robust to noise in comparison with N-PLS and PLS. If we compare Case 1 with Case 2 at different noise levels, the results revealed that the superiority of HOPLS over the other two methods was enhanced in Case 2, illustrating the advantage of HOPLS in modeling datasets with small sample size. Note that N-PLS also showed better performance than PLS when SNR$\leq$0dB in Case 2, demonstrating the advantages of modeling the dataset in a tensor form for small sample sizes. In Case 3, N-PLS showed much better performance as compared to its performance in Case 1 and Case 2, implying sensitivity of N-PLS to data distribution. With the increasing noise level, both HOPLS and N-PLS showed enhanced predictive abilities over PLS.

\subsubsection{Datasets with tensor structure}

Note that the datasets generated by (\ref{Eq:gendata-pls}) do not originally possess multi-way data structures although they were organized in a tensor form, thus the structure information of data was not important for prediction. We here assume that HOPLS is more suitable for the datasets which originally have multi-way structure, i.e. information carried by interaction among each mode are useful for our regression problem. In order to verify our assumption, the independent data $\tensor{X}$ and dependent data $\tensor{Y}$ were generated according to the Tucker model that is regarded as a general model for tensors. The latent variables $\mathbf{t}$ were generated in the same way as described in Section \ref{simulations_PLS}. A sequence of loadings $\mathbf{P}^{(n)},\mathbf{Q}^{(m)}$ and the core tensors were drawn from $\mathcal{N}(0,1)$. For the validation dataset, the latent matrix $\mathbf{T}$ was generated from the same distribution as the calibration dataset, while the core tensors and loadings were fixed. Similarly to the study in Section \ref{simulations_PLS}, to investigate how the prediction performance is affected by noise levels and sample size, $\{\tensor{X},\tensor{Y}\}\in\mathds{R}^{20\times 10\times 10}$ (Case 1) and $\{\tensor{X},\tensor{Y}\}\in\mathds{R}^{10\times 10\times 10}$ (Case 2) were generated under noise levels of 10dB, 5dB, 0dB and -5dB. The datasets for both cases were generated from five latent variables.


\begin{table}[htbp]
\caption{\small The selection of parameters $R$ and $\lambda$ in Case 2.}
\renewcommand{\arraystretch}{1.3}
\label{table:parameters2}
\centering
\begin{tabular}{c c c c c  c c c c c}
  \hline
  \multirow{2}{*}{SNR} & \multirow{2}{*}{PLS} & \multirow{2}{*}{N-PLS} & \multicolumn{2}{c}{HOPLS} & \multirow{2}{*}{SNR} & \multirow{2}{*}{PLS} & \multirow{2}{*}{N-PLS} & \multicolumn{2}{c}{HOPLS}\\
  \cline{4-5} \cline{9-10}
  & & & $R$ &$\lambda$   & & & & $R$ &$\lambda$\\
    \hline
  10dB & 5 & 7 & 9 &4  & 0dB & 4 & 4 & 4 &2 \\
  5dB & 4 & 6 & 8 &2 &-5dB & 2 & 4 & 2 &1  \\
  \hline
\end{tabular}
\end{table}

\begin{figure}[htbp]
\centering
\includegraphics[width=0.5\textwidth]{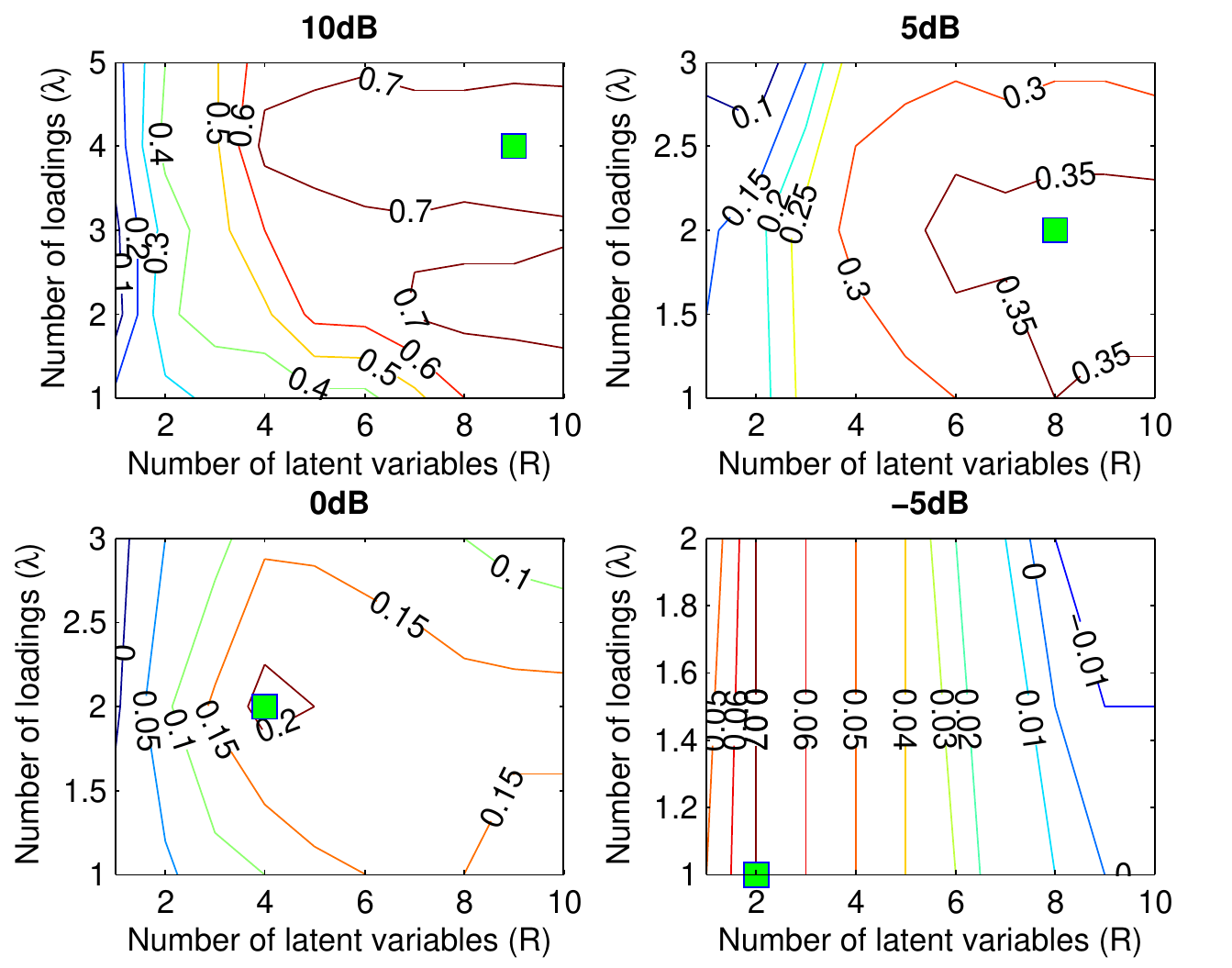}
\caption{\small Five-fold cross-validation performance of HOPLS at different noise levels versus the number of latent variables ($R$) and loadings ($\lambda$). The optimal values for these two parameters are marked by green squares.}
\label{fig:parameterscontour2}
\end{figure}

The optimal parameters of $R$ and $\lambda$ were shown in Table \ref{table:parameters2}. Observe that the optimal $R$ is smaller with the increasing noise level for all the three methods. The parameter $\lambda$ in HOPLS was also shown to have a similar behavior. For more detail, Fig. \ref{fig:parameterscontour2} exhibits the cross-validation performance grid of HOPLS with respect to $R$ and $\lambda$. When SNR was 10dB, the optimal $\lambda$ was 4, while it were 2, 2 and 1 for 5dB, 0dB and -5dB respectively. This indicates that the model complexity can be adapted to provide a better model when a specific dataset was given, demonstrating the flexibility of HOPLS model.

\begin{figure}[htbp]
\centering
\includegraphics[width=0.5\textwidth]{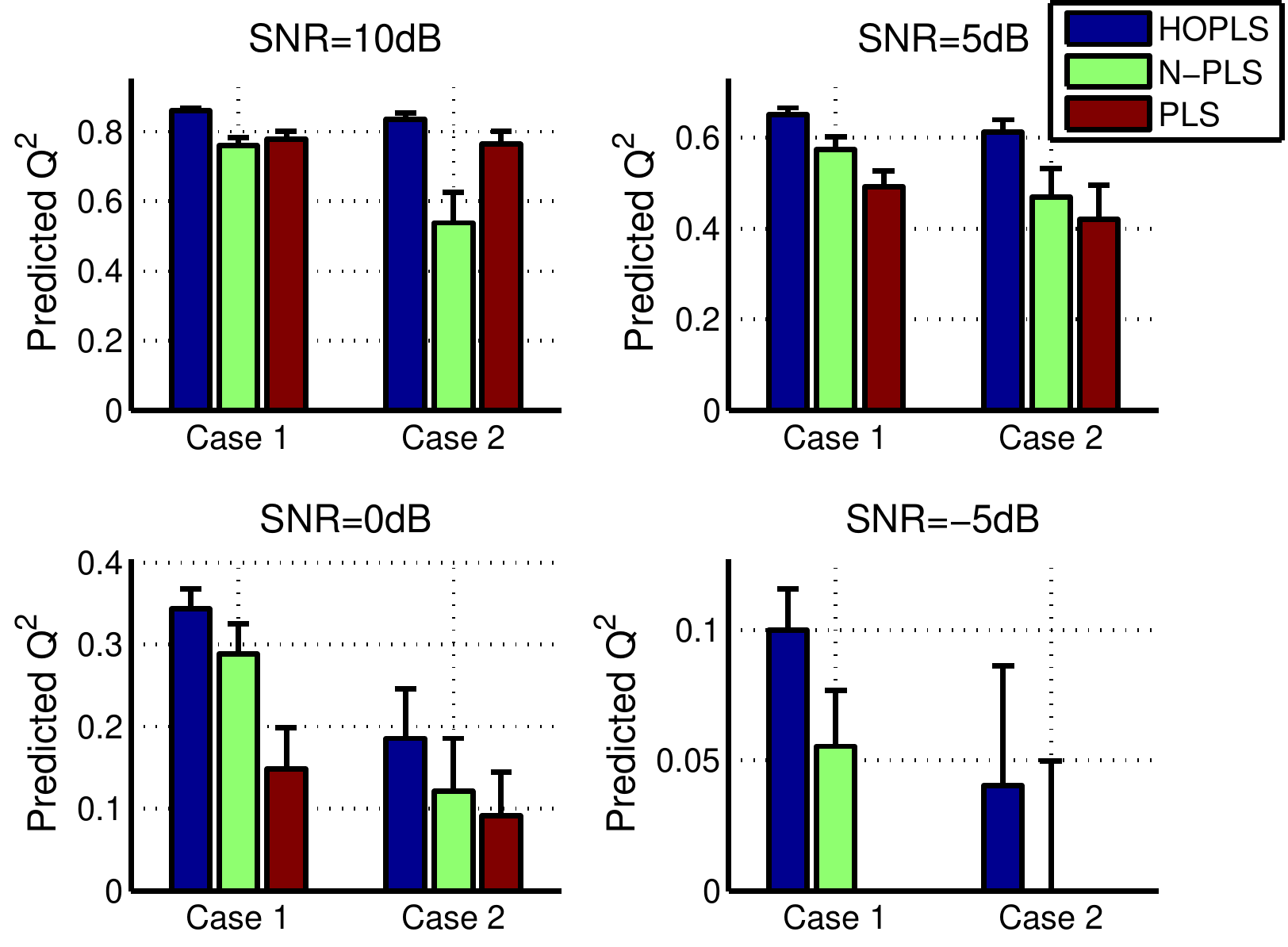}
\caption{\small The prediction performance comparison among HOPLS, N-PLS and PLS at different noise levels for the two cases (i.e., Case1: $\{\tensor{X},\tensor{Y}\}\in\mathds{R}^{20\times 10\times 10}$ and Case 2: $\{\tensor{X},\tensor{Y}\}\in\mathds{R}^{10\times 10\times 10}$) with different sample size.}
\label{fig:simulation2}
\end{figure}

The prediction performance evaluated over 50 validation datasets using HOPLS, N-PLS and PLS with individually selected parameters were compared for different noise levels and different sample sizes (i.e., two cases). As shown in Fig. \ref{fig:simulation2}, for both the cases, the prediction performance of HOPLS was better than both N-PLS and PLS at 10dB, and the discrepancy among them was enhanced when SNR changed from 10dB to -5dB. The performance of PLS decreased significantly with the increasing noise levels while HOPLS and N-PLS showed relative robustness to noise. Note that both HOPLS and N-PLS outperformed PLS when SNR$\leq$5dB, illustrating the advantages of tensor-based methods with respect to noisy data. Regarding the small sample size problem, we found the performances of all the three methods were decreased when comparing Case 1 with Case 2. Observe that the superiority of HOPLS over N-PLS and PLS were enhanced in Case 2 as compared to Case 1 at all noise levels. A comparison of Fig. \ref{fig:simulation2} and Fig. \ref{fig:simulation1} shows that the performances are significantly improved when handling the datasets having tensor structure by tensor-based methods (e.g., HOPLS and N-PLS). As for N-PLS, it outperformed PLS when the datasets have tensor structure and in the presence of high noise, but it may not perform well when the datasets have no tensor structure. By contrast, HOPLS performed well in both cases, in particular, it outperformed both N-PLS and PLS in critical cases with high noise and small sample size.

\subsubsection{Comparison on matrix response data}
In this simulation, the response data was a two-way matrix, thus HOPLS2 algorithm was used to evaluate the performance. $\underline{\mathbf{X}}\in\mathds{R}^{5\times 5\times 5\times 5}$ and $\mathbf{Y}\in\mathds{R}^{5\times 2}$ were generated from a full-rank normal distribution $\mathcal{N}(0,1)$, which satisfies $\mathbf{Y}=\mathbf{X}_{(1)}\mathbf{W}$ where $\mathbf{W}$ was also generated from $\mathcal{N}(0,1)$. Fig. \ref{fig:simulation_rand_2}(A) visualizes the predicted and original data with the red line indicating the ideal prediction. Observe that HOPLS was able to predict the validation dataset with smaller error than PLS and N-PLS. The independent data and dependent data are visualized in the latent space as shown in Fig. \ref{fig:simulation_rand_2}(B).

\begin{figure}[htbp]
\centering
\includegraphics[width=0.45\textwidth]{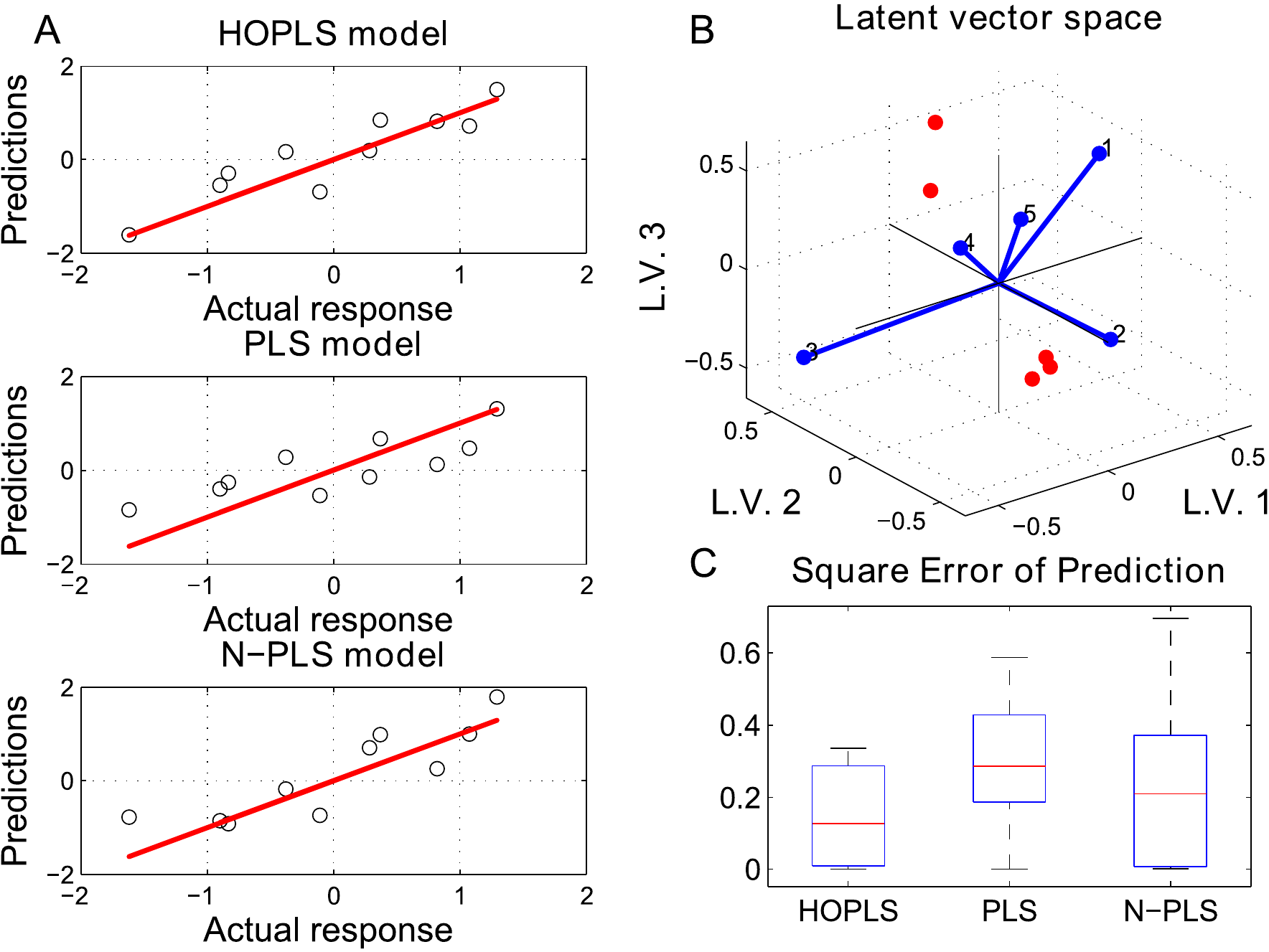}
\caption{\small (A) The scatter plot of predicted against actual data for each model. (B) Data distribution in the latent vector spaces. Each blue point denotes one sample of the independent variable, while the red points denote samples of response variables.  (C) depicts the distribution of the square error of prediction on the validation dataset.}
\label{fig:simulation_rand_2}
\end{figure}

\subsection{Decoding of ECoG signals}
In \cite{chao2010long}, ECoG-based decoding of 3D hand trajectories was demonstrated by means of classical PLS regression\footnote{The datasets and more detailed description are freely available from \url{http://neurotycho.org}.} \cite{Nagasaka2011}. The  movement of monkeys was captured by an optical motion capture system (Vicon Motion Systems, USA). In all experiments, each monkey wore a custom-made jacket with reflective markers for motion capture affixed to the left shoulder, elbows, wrists and hand, thus the response data was naturally represented as a 3th-order tensor (i.e., time $\times$ 3D positions $\times$ markers). Although PLS can be applied to predict the trajectories corresponding to each marker individually, the structure information among four markers would be unused. The ECoG data is usually transformed to the time-frequency domain in order to extract the discriminative features for decoding movement trajectories. Hence, the independent data is also naturally represented as a higher-order tensor (i.e., channel $\times$ time $\times$ frequency $\times$ samples). In this study, the proposed HOPLS regression model was applied for decoding movement trajectories based on ECoG signals to verify its effectiveness in real-world applications.

\begin{figure}[htbp]
\centering
\includegraphics[width=0.5\textwidth ]{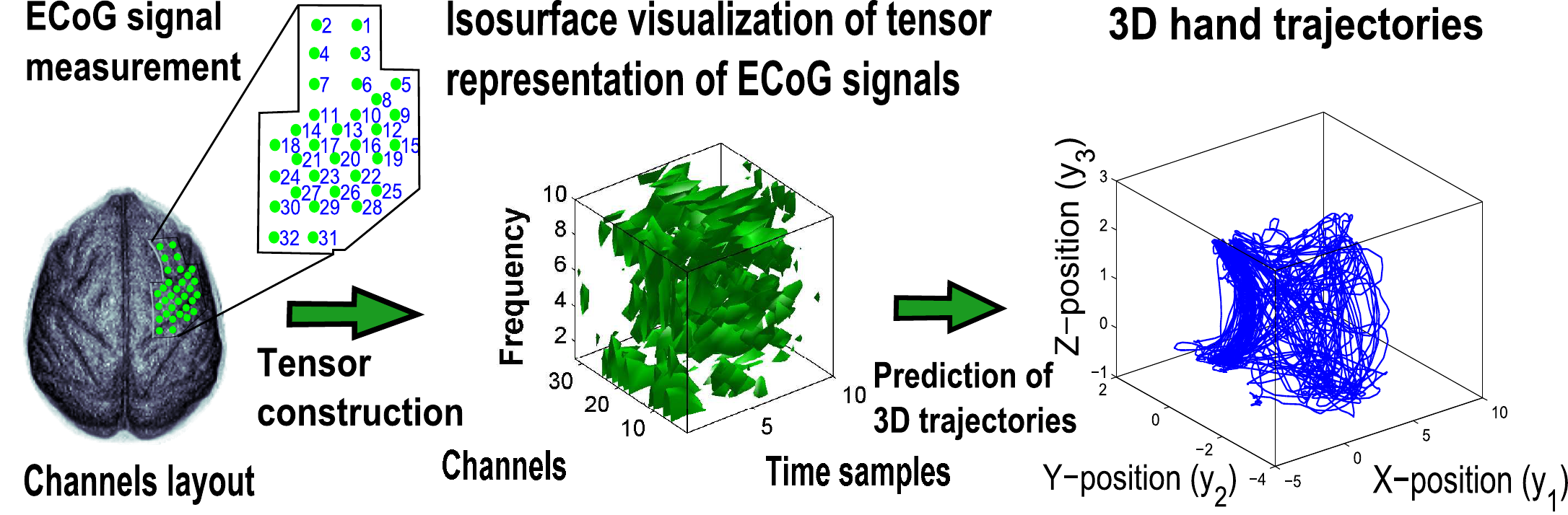}
\caption{\small The scheme for decoding of 3D hand movement trajectories from ECoG signals.}
\label{fig:DecodingScheme}
\end{figure}

The overall scheme of ECoG decoding is illustrated in Fig. \ref{fig:DecodingScheme}. Specifically, ECoG signals were preprocessed by a band-pass filter with cutoff frequencies at 0.1 and 600Hz and a spatial filter with a common average reference. Motion marker positions were down-sampled to 20Hz. In order to represent features related to the movement trajectory from ECoG signals, the Morlet wavelet transformation at 10 different center frequencies (10-150Hz, arranged in a logarithmic scale) was used to obtain the time-frequency representation. For each sample point of 3D trajectories, the most recent one-second ECoG signals were used to construct predictors. Finally, a three-order tensor of ECoG features $\underline{\mathbf{X}}\in\mathds{R}^{I_{1}\times 32 \times 100}$ (samples $\times$ channels $\times$ time-frequency) was formed to represent independent data.

\begin{figure}[htbp]
\centering
\includegraphics[width=0.45\textwidth]{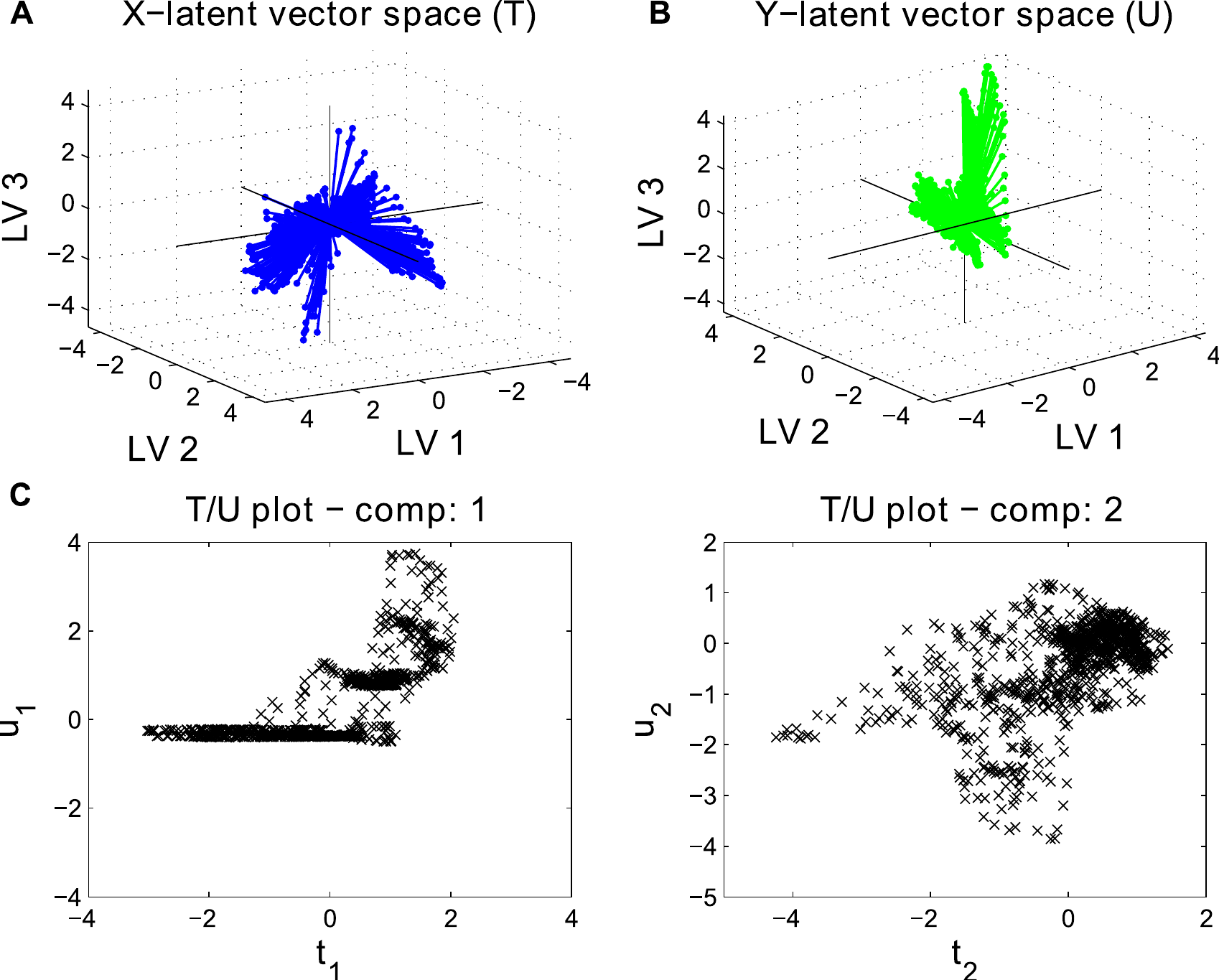}
\caption{\small Panels (A) and (B) depict data distributions in the $\underline{\mathbf{X}}$-latent space $\mathbf{T}$ and $\mathbf{Y}$-latent space $\mathbf{U}$, respectively.  (C) presents a joint distribution between $\underline{\mathbf{X}}$- and $\mathbf{Y}$-latent vectors. }
\label{fig:TandU}
\end{figure}

We first applied the HOPLS2 algorithm to predict only the hand movement trajectory, represented as a matrix $\mathbf{Y}$, for comparison with other methods. The ECoG data was divided into a calibration dataset (10 minutes) and a validation dataset (5 minutes). To select the optimal parameters of $L_n$ and $R$, the cross-validation was applied on the calibration dataset. Finally, $L_n=10$ and $R=23$ were selected for the HOPLS model. Likewise, the best values of $R$ for PLS and N-PLS were 19 and 60, respectively. The X-latent space is visualized in Fig. \ref{fig:TandU}(A), where each point represents one sample of independent variables, while the Y-latent space is presented in Fig. \ref{fig:TandU}(B), with each point representing one dependent sample. Observe that the distributions of these two latent variable spaces were quite similar, and the two dominant clusters are clearly distinguished. The joint distributions between each $\mathbf{t}_r$ and $\mathbf{u}_r$ are depicted in Fig. \ref{fig:TandU}(C). Two clusters can be observed from the first component which might be related to the `movement' and `non-movement' behaviors.

\begin{figure}[tbhp]
\centering
\includegraphics[width=0.5\textwidth]{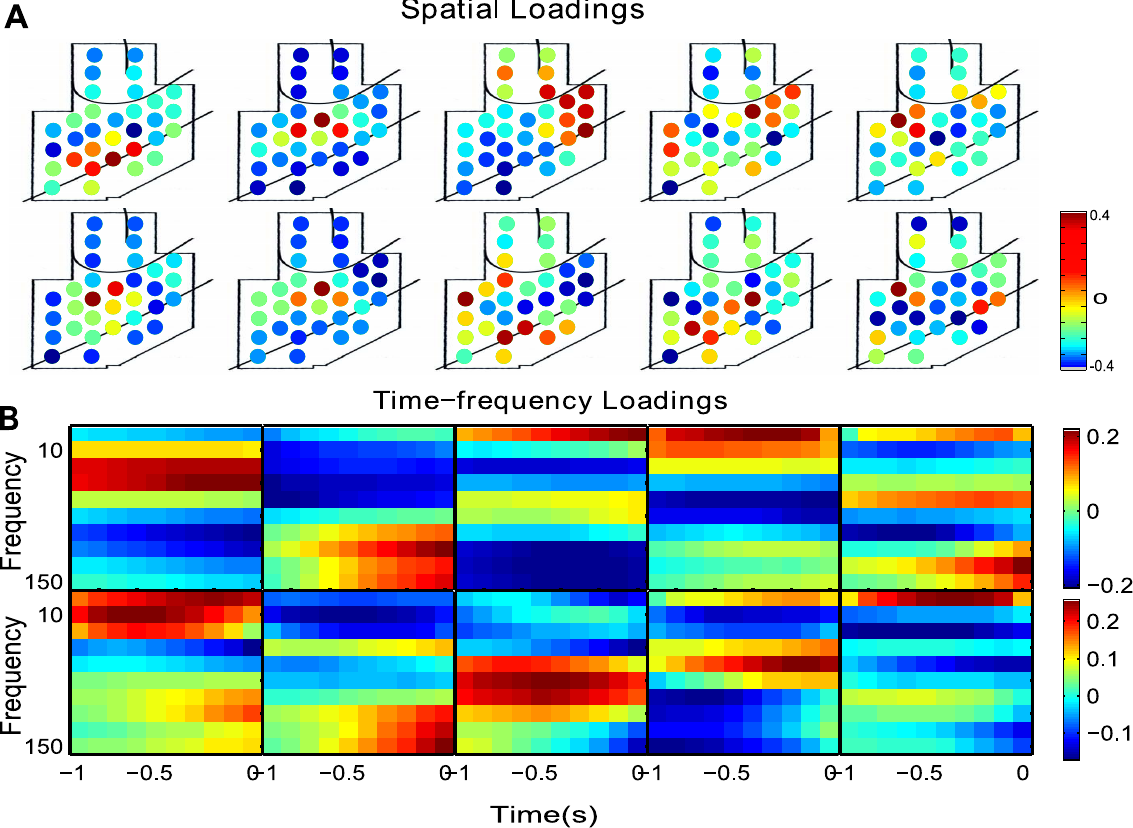}
\caption{\small (A) Spatial loadings $\mathbf{P}_r^{(1)}$ corresponding to the first two latent components. Each row shows 5 significant loading vectors.  Likewise, (B) depicts time-frequency loadings $\mathbf{P}^{(2)}_r$, with $\beta$ and $\gamma$-band exhibiting significant contribution. }
\label{fig:GTPQ}
\end{figure}

Another advantage of HOPLS was better physical interpretation of the model. To investigate how the spatial, spectral, and temporal structure of ECoG data were used to create the regression model, loading vectors can be regarded as a subspace basis in spatial and time-frequency domains, as shown in  Fig. \ref{fig:GTPQ}. With regard to time-frequency loadings, the $\beta$- and $\gamma$-band activities were most significant implying the importance of $\beta$, $\gamma$-band activities for encoding of movements; the duration of $\beta$-band was longer than that of $\gamma$-band, which indicates that hand movements were related to long history oscillations of $\beta$-band and short history oscillations of $\gamma$-band. These findings also demonstrated that a high gamma band activity in the premotor cortex is associated with movement preparation, initiation and maintenance\cite{rickert2005encoding}.

\begin{table*}[htbp]
\renewcommand{\arraystretch}{1.3}
\caption{\small Comprehensive comparison of the HOPLS, N-PLS and PLS on the prediction of 3D hand trajectories. The numbers of latent vector for HOPLS, N-PLS and Unfold-PLS were 23, 60, and 19, respectively. }
\label{table_results}
\centering
\begin{tabular}{c c c c c c c c c c c c c c c c  }
\hline\hline
\multirow{2}{*}{Data Set} & \multirow{2}{*}{Model} & \multirow{2}{*}{$Q^{2}$(ECoG)} & \multicolumn{4}{c}{$Q^{2}$(3D hand positions)} & & \multicolumn{4}{c}{RMSEP (3D hand positions)} & & \multicolumn{3}{c}{Correlation}\\
\cline{4-7}  \cline{9-12}\cline{14-16}\noalign{\smallskip}
 & & & $X$ & $Y$ & $Z$ & Mean & & $X$ & $Y$ & $Z$ & Mean  & & $X$ & $Y$ & $Z$\\
 \hline
  \multirow{3}{*}{DS1} & HOPLS  &0.25 & \textbf{0.43} & \textbf{0.48} & \textbf{0.61} &\textbf{0.51}    & & \textbf{0.82} & \textbf{0.70} & \textbf{0.66} & \textbf{0.73}& & \textbf{0.67} & \textbf{0.72} & \textbf{0.78}   \\
  & N-PLS  &0.33 &0.39 &0.44 &0.59 &0.47 &  &0.85 &0.73 &0.68 & 0.75 &  &0.64 &0.71 &0.77 \\
  & Unfold-PLS  &0.23 &0.39 & 0.45 &0.59 & 0.48 &  &0.85 &0.72 &0.68 &0.75  &  &0.64 & \textbf{0.72} &0.77  \\
  \hline
    \multirow{3}{*}{DS2} & HOPLS  &0.25 &\textbf{0.12} &\textbf{0.42} & 0.50 &\textbf{0.35} &  &\textbf{0.99} &\textbf{0.77} &0.72 &\textbf{0.83} &  &\textbf{0.35} & \textbf{0.64} &0.71  \\
  & N-PLS  &0.33 &0.03 &0.40 &0.51 &0.32 &  &1.04 &0.78 &0.71 &0.84 &  &0.32 &\textbf{0.64} &0.71  \\
  & Unfold-PLS  &0.22 &0.05 &0.40 &\textbf{0.53} & 0.32 &  &1.04 &0.78 &\textbf{0.70} &0.84 &  &0.34 &0.63 &\textbf{0.73}  \\
  \hline
    \multirow{3}{*}{DS3} & HOPLS  &0.22 &\textbf{0.36} & \textbf{0.39} & \textbf {0.48} &\textbf{0.41} &  &\textbf{0.74} & \textbf {0.77} & \textbf {0.66} & \textbf{0.73} &  & \textbf {0.62} & \textbf{0.62} & \textbf {0.69}  \\
  & N-PLS  &0.30 &0.31 &0.37 &0.46 &0.38 &  &0.77 &0.78 &0.68 &0.74 &  &0.61 & \textbf{0.62} &0.68  \\
  & Unfold-PLS  &0.21 &0.30 &{0.37} &{0.46} &0.38  &  &0.77 &0.79 &{0.67} &{0.74} &  &0.61 & \textbf{0.62} &{0.68} \\
  \hline
    \multirow{3}{*}{DS4} & HOPLS  &0.16 &\textbf{0.16} &\textbf{0.50} & \textbf{0.57} &\textbf{0.41}&  &\textbf{1.04} &\textbf{0.66} & \textbf{0.62} &\textbf{0.77} &  &\textbf{0.43} &\textbf{0.71} & \textbf{0.76}\\
  & N-PLS  &0.23 &0.12 &0.45 &0.55 &0.37 &  &1.06 &0.69 &0.67 &0.80 &  &0.41 &0.70 & \textbf{0.76}  \\
  & Unfold-PLS  &0.15 &{0.11} &0.46 & \textbf{0.57} &0.38 &  &{1.07} &0.69 & \textbf{0.62} &0.79 &  &{0.42} &0.70 & \textbf{0.76}  \\
\hline\hline
\end{tabular}
\end{table*}

From Table \ref{table_results}, observe that the improved prediction performances were achieved by HOPLS, for all the performance metrics. In particular, the results from dataset 1 demonstrated that the improvements by HOPLS over N-PLS were 0.03 for the correlation coefficient of X-position, 0.02 for averaged RMSEP, 0.04 for averaged $Q^{2}$, whereas the improvements by HOPLS over PLS were 0.03 for the correlation coefficient of X-position, 0.02 for averaged RMSEP, and 0.03 for averaged $Q^{2}$.

Since HOPLS enables us to create a regression model between two higher-order tensors, all trajectories recorded from shoulder, elbow, wrist and hand were contructed as a tensor $\tensor{Y}\in\mathds{R}^{I_{1}\times 3 \times 4}$ (samples$\times$3D positions$\times$markers). In order to verify the superiority of HOPLS for small sample sizes, we used 100 second data for calibration and 100 second data for validation. The resolution of time-frequency representations was improved to provide more detailed features, thus we have a 4th-order tensor $\tensor{X}\in\mathds{R}^{I_{1}\times 32 \times 20 \times 20}$ (samples$\times$channels$\times$ time $\times$ frequency). The prediction performances from HOPLS, N-PLS and PLS are shown in Fig. \ref{fig:ECoGperformance}, illustrating the effectiveness of HOPLS when the response data originally has tensor structure.

\begin{figure}[tbhp]
\centering
\includegraphics[width=0.45\textwidth]{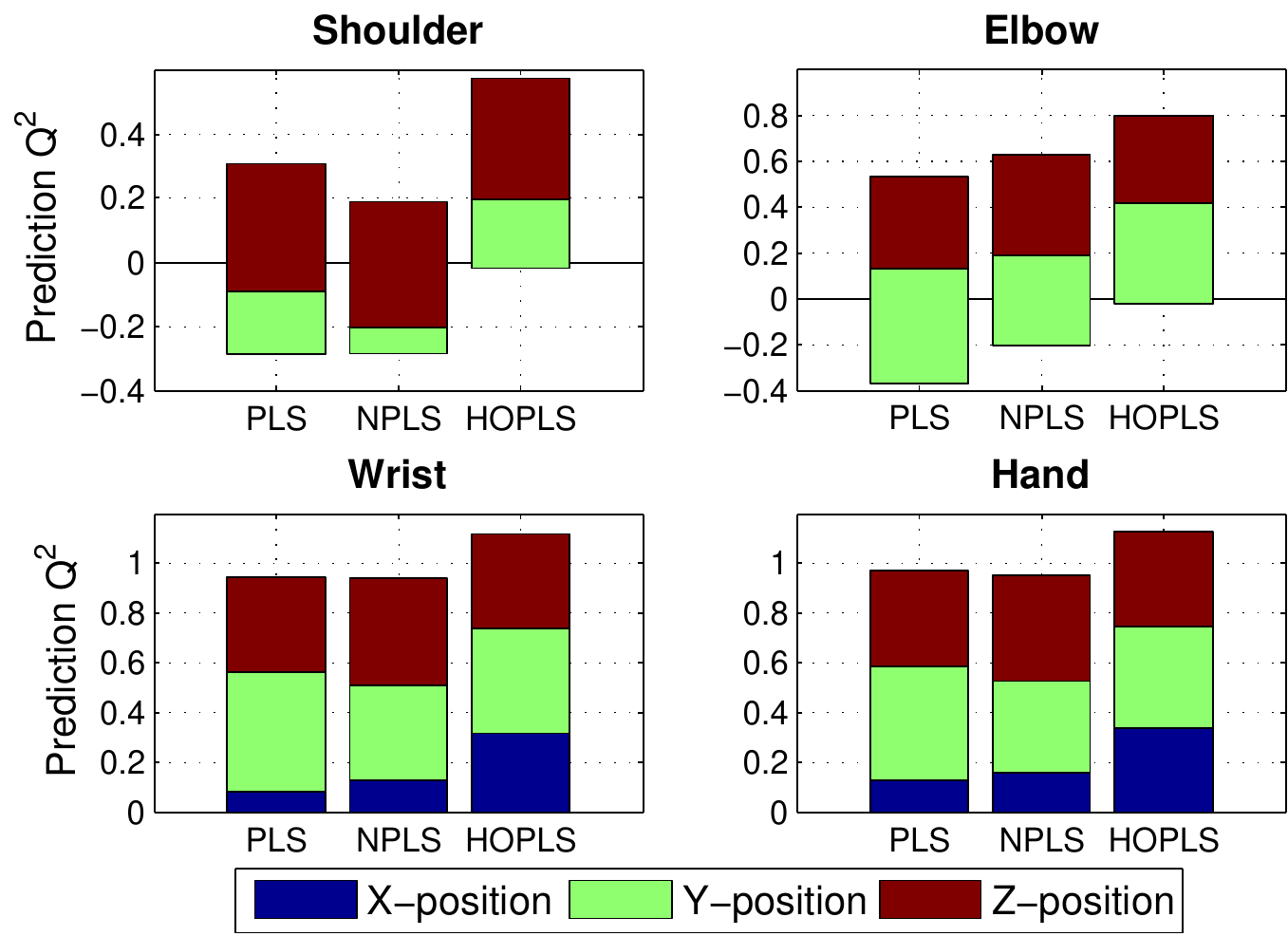}
\caption{\small The prediction performance of 3D trajectories recorded from shoulder, elbow, wrist and hand. The optimal $R$ are 16, 28, 49 for PLS, N-PLS and HOPLS, respectively, and $\lambda=5$ for HOPLS. }
\label{fig:ECoGperformance}
\end{figure}

Time-frequency features of the most recent one-second window for each sample are extremely overlapped, resulting in a lot of information redundancy and high computational burden. In addition, it is generally not necessary to predict behaviors with a high time-resolution. Hence, an additional analysis has been performed by down-sampling motion marker positions at 1Hz, to ensure that non-overlapped features were used in any adjacent samples. The cross-validation performance was evaluated for all the markers from the ten minute calibration dataset and the best performance for PLS of $Q^2=0.19$ was obtained using $R=2$, for N-PLS it was $Q^2=0.22$ obtained by $R=5$, and for HOPLS it was $Q^2=0.28$ obtained by $R=24, \lambda=5$. The prediction performances on the five minute validation dataset are shown in Fig. \ref{fig:ECoGperformanceDS20}, implying the significant improvements obtained by HOPLS over N-PLS and PLS for all the four markers. For visualization, Fig. \ref{fig:trajectory} exhibits the observed and predicted 3D hand trajectories in the 150s time window. 

\begin{figure}[tbhp]
\centering
\includegraphics[width=0.48\textwidth]{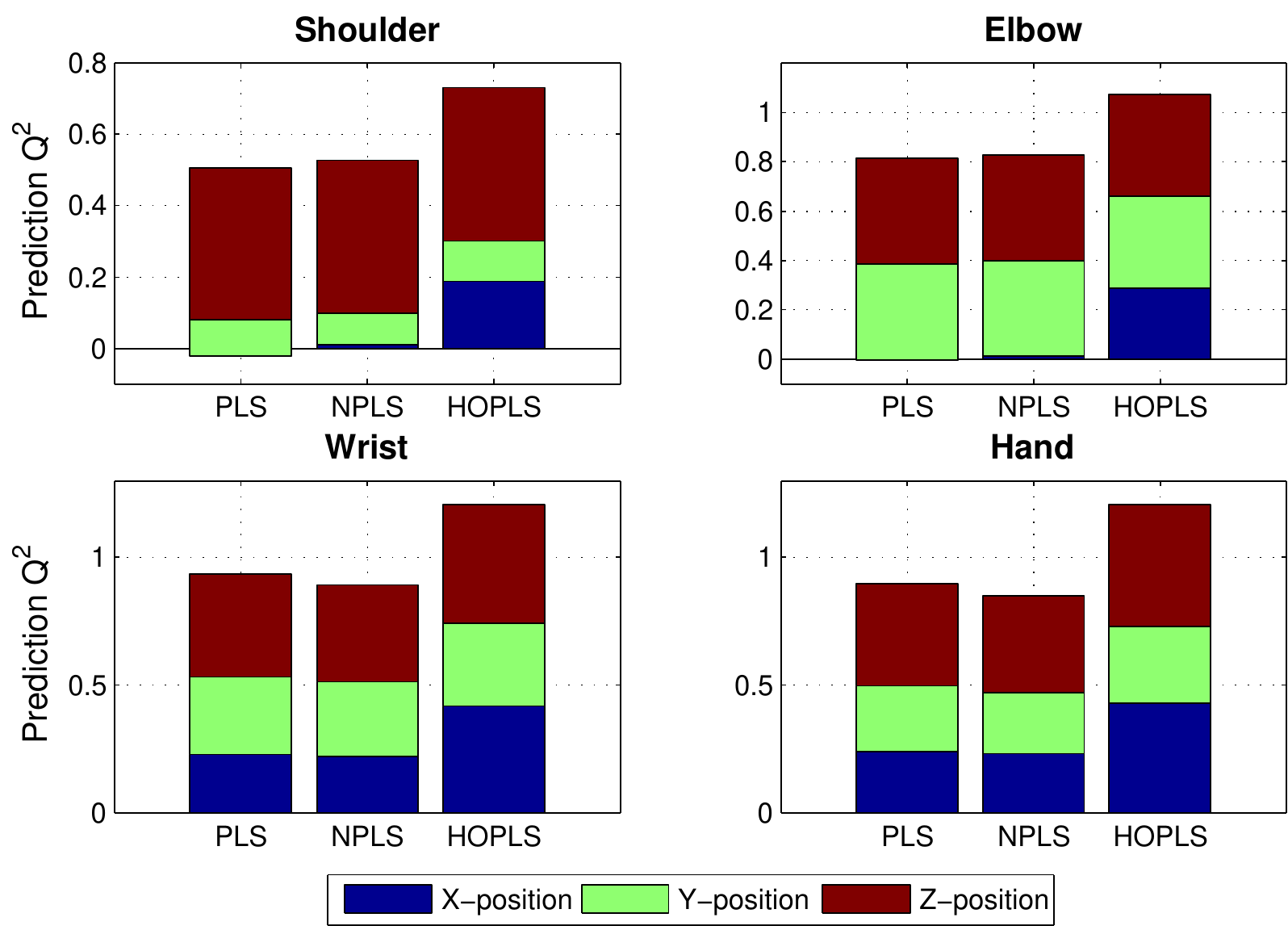}
\caption{\small The prediction performance of 3D trajectories for shoulder, elbow, wrist and hand using non-overlapped ECoG features. }
\label{fig:ECoGperformanceDS20}
\end{figure}

\begin{figure}[tbhp]
\centering
\includegraphics[width=0.5\textwidth]{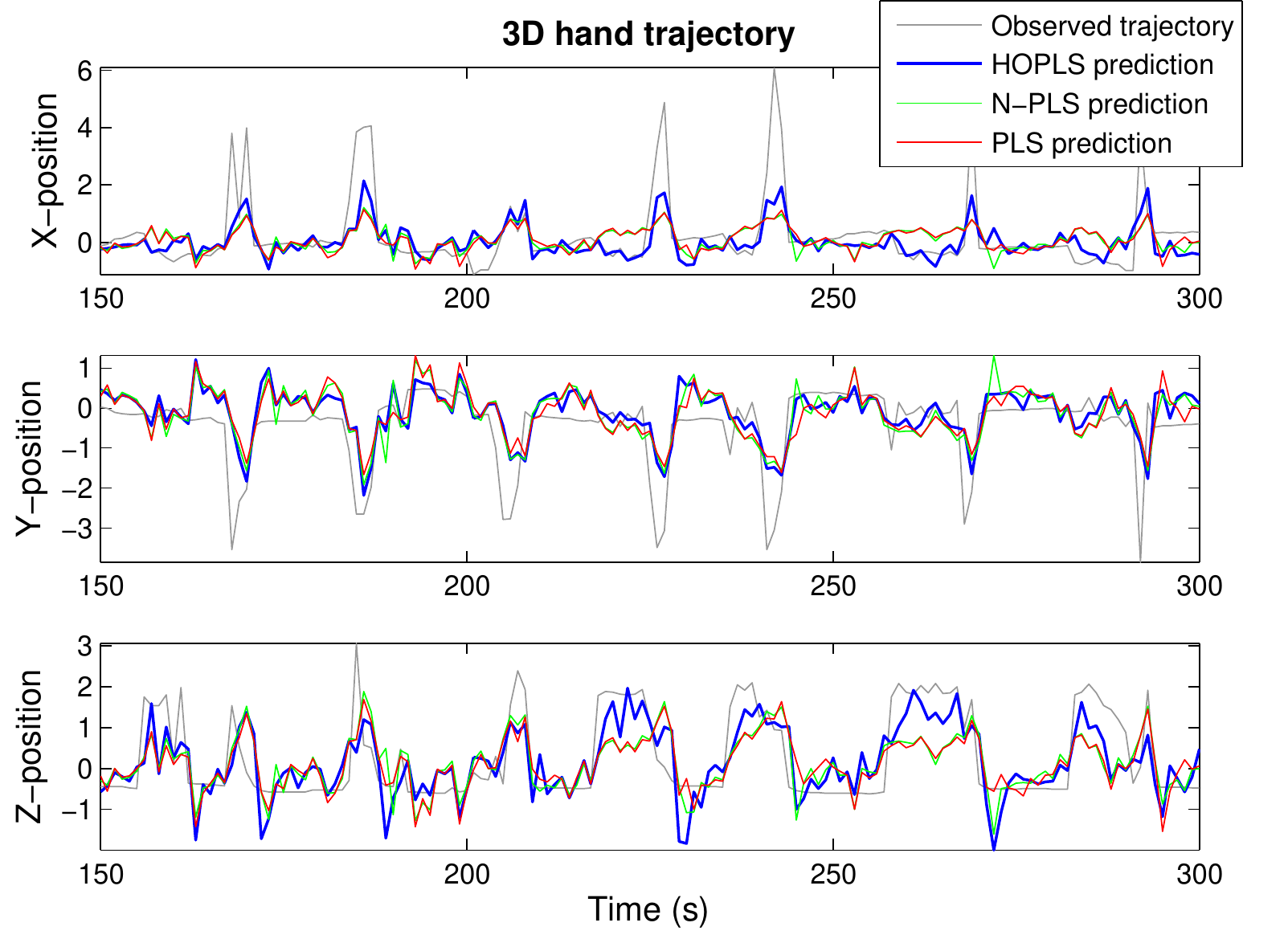}
\caption{\small Visualization of observed trajectories (150s time window) and the trajectories predicted by HOPLS, N-PLS and PLS.  }
\label{fig:trajectory}
\end{figure}

\section{Conclusions}
\label{conclusions}
The higher-order partial least squares (HOPLS) has been proposed as a generalized multilinear regression model. The analysis and simulations have shown that the advantages of the proposed model include its robustness to noise and enhanced performance for small sample sizes. In addition, HOPLS provides an optimal tradeoff between fitness and overfitting due to the fact that model complexity can be adapted by a hyperparameter. The proposed strategy to find a closed-form solution for HOPLS makes computation more efficient than the existing algorithms. The results for a real-world application in decoding 3D movement trajectories from ECoG signals have also demonstrated that HOPLS would be a promising multilinear subspace regression method.

\bibliographystyle{IEEEtran}
\bibliography{IEEEabrv,reference}

\begin{thebibliography}{10}
\providecommand{\url}[1]{#1}
\csname url@rmstyle\endcsname
\providecommand{\newblock}{\relax}
\providecommand{\bibinfo}[2]{#2}
\providecommand\BIBentrySTDinterwordspacing{\spaceskip=0pt\relax}
\providecommand\BIBentryALTinterwordstretchfactor{4}
\providecommand\BIBentryALTinterwordspacing{\spaceskip=\fontdimen2\font plus
\BIBentryALTinterwordstretchfactor\fontdimen3\font minus
  \fontdimen4\font\relax}
\providecommand\BIBforeignlanguage[2]{{%
\expandafter\ifx\csname l@#1\endcsname\relax
\typeout{** WARNING: IEEEtran.bst: No hyphenation pattern has been}%
\typeout{** loaded for the language `#1'. Using the pattern for}%
\typeout{** the default language instead.}%
\else
\language=\csname l@#1\endcsname
\fi
#2}}

\bibitem{Montgomery2001}
D.~C. Montgomery, E.~A. Peck, and G.~G. Vining, \emph{Introduction to Linear
  Regression Analysis}, 3rd~ed.\hskip 1em plus 0.5em minus 0.4em\relax New
  York: John Wiley \& Sons, 2001.

\bibitem{dhanjal2009efficient}
C.~Dhanjal, S.~Gunn, and J.~Shawe-Taylor, ``{Efficient sparse kernel feature
  extraction based on partial least squares},'' \emph{IEEE Transactions on
  Pattern Analysis and Machine Intelligence}, vol.~31, no.~8, pp. 1347--1361,
  2009.

\bibitem{wold1975soft}
H.~Wold, ``{Soft modelling by latent variables: The non-linear iterative
  partial least squares (NIPALS) approach},'' \emph{Perspectives in Probability
  and Statistics}, pp. 117--142, 1975.

\bibitem{krishnan2010partial}
A.~Krishnan, L.~Williams, A.~McIntosh, and H.~Abdi, ``{Partial least squares
  (PLS) methods for neuroimaging: A tutorial and review},'' \emph{NeuroImage},
  vol.~56, no.~2, pp. 455--475, 2010.

\bibitem{abdi2010partial}
H.~Abdi, ``{Partial least squares regression and projection on latent structure
  regression (PLS Regression)},'' \emph{Wiley Interdisciplinary Reviews:
  Computational Statistics}, vol.~2, no.~1, pp. 97--106, 2010.

\bibitem{rosipal2006overview}
R.~Rosipal and N.~Kr\"{a}mer, ``{Overview and recent advances in partial least
  squares},'' \emph{Subspace, Latent Structure and Feature Selection}, pp.
  34--51, 2006.

\bibitem{trygg2002orthogonal}
J.~Trygg and S.~Wold, ``{Orthogonal projections to latent structures
  (O-PLS)},'' \emph{Journal of Chemometrics}, vol.~16, no.~3, pp. 119--128,
  2002.

\bibitem{ergon2002pls}
R.~Ergon, ``{PLS score-loading correspondence and a bi-orthogonal
  factorization},'' \emph{Journal of Chemometrics}, vol.~16, no.~7, pp.
  368--373, 2002.

\bibitem{vijayakumar2000locally}
S.~Vijayakumar and S.~Schaal, ``{Locally weighted projection regression: An O
  (n) algorithm for incremental real time learning in high dimensional
  space},'' in \emph{Proceedings of the Seventeenth International Conference on
  Machine Learning}, vol.~1, 2000, pp. 288--293.

\bibitem{rosipal2002kernel}
R.~Rosipal and L.~Trejo, ``{Kernel partial least squares regression in
  reproducing kernel {H}ilbert space},'' \emph{The Journal of Machine Learning
  Research}, vol.~2, p. 123, 2002.

\bibitem{ham2004kernel}
J.~Ham, D.~Lee, S.~Mika, and B.~Sch\"{o}lkopf, ``{A kernel view of the
  dimensionality reduction of manifolds},'' in \emph{Proceedings of the
  twenty-first international conference on Machine learning}.\hskip 1em plus
  0.5em minus 0.4em\relax ACM, 2004, pp. 47--54.

\bibitem{bro2005standard}
R.~Bro, A.~Rinnan, and N.~Faber, ``{Standard error of prediction for
  multilinear PLS-2. Practical implementation in fluorescence spectroscopy},''
  \emph{Chemometrics and Intelligent Laboratory Systems}, vol.~75, no.~1, pp.
  69--76, 2005.

\bibitem{li2002model}
B.~Li, J.~Morris, and E.~Martin, ``{Model selection for partial least squares
  regression},'' \emph{Chemometrics and Intelligent Laboratory Systems},
  vol.~64, no.~1, pp. 79--89, 2002.

\bibitem{tibshirani1996regression}
R.~Tibshirani, ``Regression shrinkage and selection via the lasso,''
  \emph{Journal of the Royal Statistical Society. Series B (Methodological)},
  pp. 267--288, 1996.

\bibitem{bro1998multi}
R.~Bro, ``{Multi-way analysis in the food industry},'' \emph{Models,
  Algorithms, and Applications. Academish proefschrift. Dinamarca}, 1998.

\bibitem{kolda2009tensor}
T.~Kolda and B.~Bader, ``{Tensor Decompositions and Applications},'' \emph{SIAM
  Review}, vol.~51, no.~3, pp. 455--500, 2009.

\bibitem{Cichocki2009}
A.~Cichocki, R.~Zdunek, A.~H. Phan, and S.~I. Amari, \emph{{N}onnegative
  {M}atrix and {T}ensor {F}actorizations}.\hskip 1em plus 0.5em minus
  0.4em\relax John Wiley \& Sons, 2009.

\bibitem{acar2010scalable}
E.~Acar, D.~Dunlavy, T.~Kolda, and M.~M{\o}rup, ``{Scalable tensor
  factorizations for incomplete data},'' \emph{Chemometrics and Intelligent
  Laboratory Systems}, 2010.

\bibitem{bro1996multiway}
R.~Bro, ``{Multiway calibration. Multilinear PLS},'' \emph{Journal of
  Chemometrics}, vol.~10, no.~1, pp. 47--61, 1996.

\bibitem{bro2006review}
------, ``{Review on multiway analysis in chemistryÑ2000--2005},''
  \emph{Critical Reviews in Analytical Chemistry}, vol.~36, no.~3, pp.
  279--293, 2006.

\bibitem{hasegawa2000rational}
K.~Hasegawa, M.~Arakawa, and K.~Funatsu, ``{Rational choice of bioactive
  conformations through use of conformation analysis and 3-way partial least
  squares modeling},'' \emph{Chemometrics and Intelligent Laboratory Systems},
  vol.~50, no.~2, pp. 253--261, 2000.

\bibitem{nilsson1997multiway}
J.~Nilsson, S.~de~Jong, and A.~Smilde, ``{Multiway calibration in 3D QSAR},''
  \emph{Journal of chemometrics}, vol.~11, no.~6, pp. 511--524, 1997.

\bibitem{zissis1999two}
K.~Zissis, R.~Brereton, S.~Dunkerley, and R.~Escott, ``{Two-way, unfolded
  three-way and three-mode partial least squares calibration of diode array
  HPLC chromatograms for the quantitation of low-level pharmaceutical
  impurities},'' \emph{Analytica Chimica Acta}, vol. 384, no.~1, pp. 71--81,
  1999.

\bibitem{martinez2004concurrent}
E.~Martinez-Montes, P.~Vald{\'e}s-Sosa, F.~Miwakeichi, R.~Goldman, and
  M.~Cohen, ``{Concurrent EEG/fMRI analysis by multiway partial least
  squares},'' \emph{NeuroImage}, vol.~22, no.~3, pp. 1023--1034, 2004.

\bibitem{acar2007seizure}
E.~Acar, C.~Bingol, H.~Bingol, R.~Bro, and B.~Yener, ``{Seizure recognition on
  epilepsy feature tensor},'' in \emph{29th Annual International Conference of
  the IEEE EMBS 2007}, 2007, pp. 4273--4276.

\bibitem{bro2001difference}
R.~Bro, A.~Smilde, and S.~de~Jong, ``{On the difference between low-rank and
  subspace approximation: improved model for multi-linear PLS regression},''
  \emph{Chemometrics and Intelligent Laboratory Systems}, vol.~58, no.~1, pp.
  3--13, 2001.

\bibitem{smilde1999multiway}
A.~Smilde and H.~Kiers, ``{Multiway covariates regression models},''
  \emph{Journal of Chemometrics}, vol.~13, no.~1, pp. 31--48, 1999.

\bibitem{smilde2004multi}
A.~Smilde, R.~Bro, and P.~Geladi, \emph{{Multi-way analysis with applications
  in the chemical sciences}}.\hskip 1em plus 0.5em minus 0.4em\relax Wiley,
  2004.

\bibitem{Harshman1970}
R.~A. Harshman, ``Foundations of the {PARAFAC} procedure: Models and conditions
  for an explanatory multimodal factor analysis,'' \emph{UCLA Working Papers in
  Phonetics}, vol.~16, pp. 1--84, 1970.

\bibitem{smilde1997comments}
A.~Smilde, ``{Comments on multilinear PLS},'' \emph{Journal of Chemometrics},
  vol.~11, no.~5, pp. 367--377, 1997.

\bibitem{B903649K}
M.~D. Borraccetti, P.~C. Damiani, and A.~C. Olivieri, ``When unfolding is
  better: unique success of unfolded partial least-squares regression with
  residual bilinearization for the processing of spectral-p{H} data with strong
  spectral overlapping. analysis of fluoroquinolones in human urine based on
  flow-injection p{H}-modulated synchronous fluorescence data matrices,''
  \emph{Analyst}, vol. 134, pp. 1682--1691, 2009.

\bibitem{de2008decompositions}
L.~De~Lathauwer, ``Decompositions of a higher-order tensor in block terms -
  {P}art {II}: Definitions and uniqueness,'' \emph{SIAM J. Matrix Anal. Appl},
  vol.~30, no.~3, pp. 1033--1066, 2008.

\bibitem{de2000multilinear}
L.~De~Lathauwer, B.~De~Moor, and J.~Vandewalle, ``{A multilinear singular value
  decomposition},'' \emph{SIAM Journal on Matrix Analysis and Applications},
  vol.~21, no.~4, pp. 1253--1278, 2000.

\bibitem{Tucker1963}
L.~R. Tucker, ``Implications of factor analysis of three-way matrices for
  measurement of change,'' in \emph{Problems in {M}easuring {C}hange}, C.~W.
  Harris, Ed.\hskip 1em plus 0.5em minus 0.4em\relax University of Wisconsin
  Press, 1963, pp. 122--137.

\bibitem{kolda2006multilinear}
T.~Kolda, \emph{Multilinear operators for higher-order decompositions}.\hskip
  1em plus 0.5em minus 0.4em\relax Tech. Report SAND2006-2081, Sandia National
  Laboratories, Albuquerque, NM, Livermore, CA, 2006.

\bibitem{Carroll1970}
J.~D. Carroll and J.~J. Chang, ``Analysis of individual differences in
  multidimensional scaling via an \textsc{N}-way generalization of
  \textquotedblleft \textsc{E}ckart-\textsc{Y}oung\textquotedblright
  decomposition,'' \emph{Psychometrika}, vol.~35, pp. 283--319, 1970.

\bibitem{de2000best}
L.~De~Lathauwer, B.~De~Moor, and J.~Vandewalle, ``{On the Best Rank-1 and
  Rank-(R1, R2,..., RN) Approximation of Higher-Order Tensors},'' \emph{SIAM
  Journal on Matrix Analysis and Applications}, vol.~21, no.~4, pp. 1324--1342,
  2000.

\bibitem{Silva2008}
L.-H. Lim and V.~D. Silva, ``Tensor rank and the ill-posedness of the best
  low-rank approximation problem,'' \emph{SIAM Journal of Matrix Analysis and
  Applications}, vol.~30, no.~3, pp. 1084--1127, 2008.

\bibitem{wold1982soft}
H.~Wold, ``{Soft modeling: the basic design and some extensions},''
  \emph{Systems Under Indirect Observation}, vol.~2, pp. 1--53, 1982.

\bibitem{wold2001pls}
S.~Wold, M.~Sjostroma, and L.~Erikssonb, ``{PLS-regression: a basic tool of
  chemometrics},'' \emph{Chemometrics and Intelligent Laboratory Systems},
  vol.~58, pp. 109--130, 2001.

\bibitem{wold1984collinearity}
S.~Wold, A.~Ruhe, H.~Wold, and W.~Dunn~III, ``{The collinearity problem in
  linear regression. The partial least squares (PLS) approach to generalized
  inverses},'' \emph{SIAM Journal on Scientific and Statistical Computing},
  vol.~5, p. 735, 1984.

\bibitem{kowalski1982systems}
B.~Kowalski, R.~Gerlach, and H.~Wold, ``{Systems under indirect observation},''
  \emph{Chemical Systems under Indirect Observation}, pp. 191--209, 1982.

\bibitem{mcintosh2004partial}
A.~McIntosh and N.~Lobaugh, ``{Partial least squares analysis of neuroimaging
  data: applications and advances},'' \emph{Neuroimage}, vol.~23, pp.
  S250--S263, 2004.

\bibitem{mcintosh2004spatiotemporal}
A.~McIntosh, W.~Chau, and A.~Protzner, ``{Spatiotemporal analysis of
  event-related fMRI data using partial least squares},'' \emph{Neuroimage},
  vol.~23, no.~2, pp. 764--775, 2004.

\bibitem{kovacevic2007groupwise}
N.~Kovacevic and A.~McIntosh, ``{Groupwise independent component decomposition
  of EEG data and partial least square analysis},'' \emph{NeuroImage}, vol.~35,
  no.~3, pp. 1103--1112, 2007.

\bibitem{chao2010long}
Z.~Chao, Y.~Nagasaka, and N.~Fujii, ``{Long-term asynchronous decoding of arm
  motion using electrocorticographic signals in monkeys},'' \emph{Frontiers in
  Neuroengineering}, vol.~3, no.~3, 2010.

\bibitem{trejo2006brain}
L.~Trejo, R.~Rosipal, and B.~Matthews, ``{Brain-computer interfaces for 1-D and
  2-D cursor control: designs using volitional control of the EEG spectrum or
  steady-state visual evoked potentials},'' \emph{IEEE Transactions on Neural
  Systems and Rehabilitation Engineering}, vol.~14, no.~2, pp. 225--229, 2006.

\bibitem{kim2005three}
H.~Kim, J.~Zhou, H.~Morse~III, and H.~Park, ``{A three-stage framework for gene
  expression data analysis by L1-norm support vector regression},''
  \emph{International Journal of Bioinformatics Research and Applications},
  vol.~1, no.~1, pp. 51--62, 2005.

\bibitem{Nagasaka2011}
Y.~Nagasaka, K.~Shimoda, and N.~Fujii, ``Multidimensional recording ({MDR}) and
  data sharing: An ecological open research and educational platform for
  neuroscience,'' \emph{PLoS ONE}, vol.~6, no.~7, p. e22561, 2011.

\bibitem{rickert2005encoding}
J.~Rickert, S.~de~Oliveira, E.~Vaadia, A.~Aertsen, S.~Rotter, and C.~Mehring,
  ``Encoding of movement direction in different frequency ranges of motor
  cortical local field potentials,'' \emph{The Journal of Neuroscience},
  vol.~25, no.~39, pp. 8815--8824, 2005.

\end{thebibliography}
\vfill\vfill\vfill\vfill\vfill\vfill






\end{document}